\documentclass{article}

\PassOptionsToPackage{numbers, compress}{natbib}

\usepackage[main, final]{neurips_2026}

\usepackage[utf8]{inputenc} 
\usepackage[T1]{fontenc}    
\usepackage{hyperref}       
\usepackage{url}            
\usepackage{booktabs}       
\usepackage{amsfonts}       
\usepackage{nicefrac}       
\usepackage{microtype}      
\usepackage{xcolor}
\newif\ifcrchanges
\crchangestrue

\usepackage[percent]{overpic}
\usepackage{graphicx}
\usepackage{multirow}
\usepackage{amsmath}
\usepackage{colortbl}
\usepackage{amssymb}
\usepackage{makecell}
\usepackage{tabularx}
\usepackage{adjustbox}
\usepackage{arydshln}
\usepackage{amsthm}
\usepackage{mathtools}
\usepackage{wrapfig}
\usepackage{float}
\usepackage{algorithm}
\usepackage{algpseudocode}
\usepackage{subfig}    

\theoremstyle{plain}
\newtheorem{theorem}{Theorem}[section]
\newtheorem{proposition}[theorem]{Proposition}

\theoremstyle{definition}

\theoremstyle{remark}

\usepackage{pifont}

\newcommand{\std}[1]{\text{\scriptsize\textcolor{gray}{\ensuremath{\pm#1}}}}

\definecolor{ourcolor}{HTML}{99e0eb}
\definecolor{ourblue}{HTML}{27a2c3}

\definecolor{tablecolor}{HTML}{d6e8f6} 

\definecolor{tablecolor2}{HTML}{ffcdb4}
\definecolor{citecolor}{HTML}{fe7b5b}
\definecolor{grey}{rgb}{0.9, 0.9, 0.9}

\definecolor{drift}{HTML}{d6f6ee}
\definecolor{flowmap}{HTML}{d6e8f6}
\definecolor{multistep}{HTML}{ead6f6}

\newcommand{\dd}[2]{#1\std{#2}}
\newcommand{\ddbf}[2]{\cellcolor{tablecolor}\textbf{#1}\std{#2}}

\makeatletter
\def\adl@drawiv#1#2#3{%
        \hskip.5\tabcolsep
        \xleaders#3{#2.5\@tempdimb #1{1}#2.5\@tempdimb}%
                #2\z@ plus1fil minus1fil\relax
        \hskip.5\tabcolsep}
\newcommand{\cdashlinelr}[1]{%
  \noalign{\vskip\aboverulesep
           \global\let\@dashdrawstore\adl@draw
           \global\let\adl@draw\adl@drawiv}
  \cdashline{#1}
  \noalign{\global\let\adl@draw\@dashdrawstore
           \vskip\belowrulesep}}
\makeatother

\title{Implicit Drifting Policy: One-Step Action Generation via Conditional Expert Geometry}

\author{%
  Zemin Yang$^{1}$, Yaoyu He$^{1}$, Yiming Zhong$^{1}$, Yuhao Zhang$^{2}$, Xinge Zhu$^{3}$\\
  \textbf{Yao Mu$^{2}$,} 
  \textbf{Qingqiu Huang$^{4}$,} 
  \textbf{Yuexin Ma$^{1}$\thanks{Corresponding author. \vspace{-10pt}}} \\
  $^{1}$ShanghaiTech University \quad $^{2}$Shanghai Jiao Tong University \\
  $^{3}$The Chinese University of Hong Kong \quad
  $^{4}$Morphi Robot \\
  \texttt{\{csyangzm, mayuexin\}@shanghaitech.edu.cn} \\
  \textbf{Project Page: } \url{https://implicit-drifting-policy.github.io/} \\
  \vspace{-25pt}
}

\begin{document}

\maketitle

\vspace{-5pt}
\begin{abstract}
  \vspace{-5pt}
  Generative action policies based on diffusion or flow matching excel in behavior cloning, 
  yet their iterative sampling is prohibitive for high-frequency robot control. 
  While recent one-step formulations alleviate this latency, 
  they inevitably discard the intermediate trajectory evolution that provides crucial action correction. 
  Directly recovering this mechanism by explicitly estimating a training-time drifting field is challenging under sparse conditional demonstrations. 
  We introduce \textbf{Implicit Drifting Policy (IDP)}, 
  a one-step imitation learning framework that brings the training-time correction of Drifting into policy learning without explicit vector field estimation. 
  IDP extracts a \textbf{conditional expert geometry} from the local variation of observation-similar expert actions, 
  and compares it against a global \textbf{reference geometry} to isolate condition-specific constraints. 
  This local geometric structure adaptively weights a scalar potential objective. 
  Combined with an expert-proximal terminal evaluation, 
  IDP incorporates local geometric correction into the one-step generator during training. 
  Extensive evaluations across 2D, 3D, and real-world manipulation tasks show IDP effectively maintains adherence to valid action manifolds, 
  improving upon explicit drifting methods and achieving competitive performance with strong one-step baselines.
\end{abstract}


\section{Introduction}

Robot imitation learning maps observations to continuous actions. 
To achieve this, behavior cloning commonly trains a policy with a regression loss to predict an expert action or action sequence from an observation. 
This form is simple and efficient, but it compresses action generation into a single prediction. 
Diffusion Policy introduced diffusion models into this setting by formulating action prediction as conditional denoising, gradually generating actions from noise toward expert actions~\cite{chi2024diffusionpolicyvisuomotorpolicy}. 
Subsequent flow-based policies further describe action generation with continuous flows or flow matching~\cite{zhang2024flowpolicyenablingfastrobust}. 
Compared with direct prediction, diffusion and flow policies retain intermediate states during generation, allowing an action to be formed through multiple updates. 
This intermediate computation is one of the key distinctions between generative policies and regression-based policies.

The inference cost of generative policies has motivated faster action generation methods. 
Consistency Policy and One-Step Diffusion Policy compress diffusion policies through consistency training or distillation~\cite{prasad2024consistencypolicyacceleratedvisuomotor,wang2024onestepdiffusionpolicyfast}. 
Shortcut Models and flow-map methods learn longer-range generative maps to reduce sampling steps~\cite{frans2025stepdiffusionshortcutmodels,li2026longshortflowmapperspectivedrifting}. 
Beyond sampling speed, a more basic question is what multi-step generation contributes to control. 
Much Ado About Noising analyzes generative robotic control from this perspective and argues that its benefits are closely tied to stochasticity injection and supervised intermediate computation, rather than multimodality alone~\cite{pan2026adonoisingdispellingmyths}. 
This raises a question for one-step policies: \textit{when inference no longer contains multi-step updates, how can training still provide a useful action-correction signal?}

\begin{figure}[t]
\centering
\includegraphics[width=0.8\linewidth]{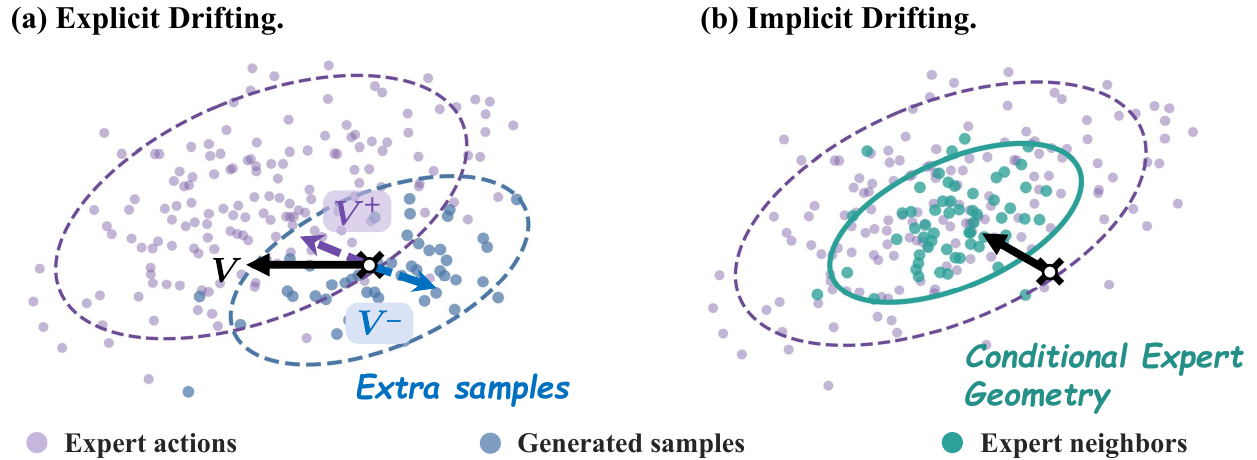}
\caption{%
\textbf{(a) Explicit Drifting} constructs targets by estimating a drifting field between expert actions (purple) and policy predictions (blue). Relying on the evolving policy state and mini-batch sampling can make the target sensitive to minibatch composition and the current policy state. 
\textbf{(b) Implicit Drifting Policy (Ours)} directly extracts a Conditional Expert Geometry (green) from expert actions under similar observations. 
}
\vspace{-2.5ex}
\label{fig:teaser}
\end{figure}

Drifting models study a related mechanism in generative modeling. 
Drifting trains a single-pass generator and uses a drifting field during training to move the current model distribution toward the data distribution~\cite{deng2026generativemodelingdrifting}. 
Recent analyses connect this field to score-based transport directions and terminal data-grounded correction~\cite{lai2026unifiedviewdriftingscorebased,li2026longshortflowmapperspectivedrifting}. 
Here, correction affects the learning dynamics while inference remains a single generator evaluation. 
This makes Drifting relevant to one-step action policies: the training objective itself can carry part of the action-correction role, reducing reliance on additional sampling steps at deployment. 
However, behavior cloning provides a different form of supervision. 
For a given observation, demonstrations usually provide a single expert action. 
Estimating a drifting field around the policy output thus requires batch neighbors and the prediction location to approximate local conditional structure. 
Such estimates vary with the mini-batch and evolving policy state, making the target sensitive to implementation details (Fig~\ref{fig:teaser} (a)).

We argue that a reliable correction signal should instead emerge from the local structure already present in demonstrations. 
A single observation provides one expert action, while expert actions under similar observations provide local references. 
Similar observations often correspond to nearby task states or control contexts. 
Within these local references, variation among expert actions indicates which action directions are more consistent near the current condition and which directions allow more variation. 
We use this demonstration-supported information to build a training signal that emphasizes errors along locally important action directions, 
without estimating a drifting field as a separate supervision target. 
This keeps the training-time correction view of Drifting while grounding the correction source in local action structure observable in imitation learning.

Motivated by this, we propose \textbf{Implicit Drifting Policy} (IDP). 
IDP formulates an implicit correction objective that minimizes a geometry-aware potential. 
The core of this potential lies in its adaptive directional weights, which we derive by extracting a \textbf{conditional expert geometry} from the local variation of expert actions under similar observations (Fig~\ref{fig:teaser} (b)). 
To isolate condition-specific constraints, we further compare this local geometry with a global \textbf{reference geometry}, ensuring that inherently low-variance dimensions are not mistakenly penalized. 
Finally, to ensure the network effectively captures the local manifold topology defined by this potential, IDP introduces an expert-proximal evaluation during training, exposing the one-step generator to the local correction structure near expert actions. 
Across 2D, 3D imitation tasks, and real-world robot experiments, IDP improves over explicit drifting variants and remains competitive with strong one-step policy baselines.
In summary, our contributions are as follows:
\begin{itemize}
    \item We introduce \textbf{Implicit Drifting Policy (IDP)}, a one-step imitation learning framework. By minimizing a geometry-aware potential and introducing an expert-proximal training evaluation, IDP brings the training-time correction of drifting into policy learning.
    \item We propose a method to extract local action correction structures by comparing the \textbf{conditional expert geometry} under similar observations against a global \textbf{reference geometry}. This adaptively isolates condition-specific constraints to weight the correction potential.
    \item We evaluate IDP across 2D, 3D, and real-world manipulation tasks. Results demonstrate that IDP outperforms explicit drifting variants and exhibits stronger adherence to local action constraints over strong one-step baselines.
\end{itemize}

\section{Related Work}

\paragraph{Generative action policies and one-step acceleration.}
Continuous-action imitation learning commonly uses behavior cloning, where a policy predicts expert actions from observations.
Diffusion Policy formulates action prediction as conditional denoising, so that actions are formed through multiple denoising steps~\cite{chi2024diffusionpolicyvisuomotorpolicy}.
Flow matching provides a continuous-time transport formulation for generative modeling~\cite{lipman2023flowmatchinggenerativemodeling}, and has been used for policies with point-cloud inputs, spatially structured action representations, 3D manipulation, and action-to-action generation~\cite{chisari2024learningroboticmanipulationpolicies,funk2024actionflowequivariantaccurateefficient,zhang2024flowpolicyenablingfastrobust,jia2026actiontoactionflowmatching}.
Recent work reduces the inference cost of generative policies from several directions.
Consistency Policy~\cite{prasad2024consistencypolicyacceleratedvisuomotor} and ManiCM~\cite{lu2025manicmrealtime3ddiffusion} use consistency models to accelerate diffusion policies.
One-Step Diffusion Policy~\cite{wang2024onestepdiffusionpolicyfast}, One-Step Flow Policy~\cite{li2026onestepflowpolicyselfdistillation}, and MP1~\cite{sheng2025mp1meanflowtamespolicy} obtain single-step policies through distillation, self-distillation, or average-velocity learning.
Streaming Diffusion Policy~\cite{hoeg2024streamingdiffusionpolicyfast}, Falcon~\cite{chen2025falconfastvisuomotorpolicies}, and Real-Time Iteration Scheme~\cite{duan2025realtimeiterationschemediffusion} reduce, reorganize, or reuse the denoising process to lower latency.
In broader generative modeling, Shortcut Models~\cite{frans2025stepdiffusionshortcutmodels}, Mean Flows~\cite{geng2025meanflowsonestepgenerative}, flow-map distillation~\cite{boffi2025buildconsistencymodellearning}, and Align Your Flow~\cite{sabour2025alignflowscalingcontinuoustime} study long-interval or one-step generative mappings.
Much Ado About Noising~\cite{pan2026adonoisingdispellingmyths} moves the discussion from sampling steps to mechanism, arguing that the benefit of generative control policies is closely tied to stochasticity injection and supervised intermediate computation.
Our work follows this concern for one-step policy training and asks how training can retain an action-correction signal once deployment uses a single policy evaluation.

\paragraph{Drifting models and training-time correction.}
Drifting models study one-step generative modeling from the perspective of training-time distribution evolution, where a drifting field describes the direction that moves the current model distribution toward the data distribution.
Generative Modeling via Drifting~\cite{deng2026generativemodelingdrifting} trains a single-pass generator and uses this drifting field during training.
Subsequent analysis relates the drifting field to score-based transport directions~\cite{lai2026unifiedviewdriftingscorebased}.
Long-Short Flow-Map~\cite{li2026longshortflowmapperspectivedrifting} further interprets data-grounded terminal correction as supervision for a long map.
These works provide a view of training-time correction for one-step generation, primarily in distribution-level generative modeling.
Behavior cloning provides a more specific form of supervision. Since demonstrations usually provide only a single deterministic expert action for a given observation, estimating a continuous drifting field around the current policy output is forced to depend on local mini-batch neighborhoods and the evolving prediction location.
We build on the training-time correction view of Drifting while grounding the correction signal in local information available in imitation learning.

\paragraph{Geometric structure in imitation learning.}
Imitation learning has long used demonstrations to represent action distributions, trajectory structure, and variability.
Gaussian mixture regression, dynamic movement primitives, probabilistic movement primitives, kernelized movement primitives, and stable dynamical systems learn local means, covariances, temporal evolution, or stable control fields from demonstration trajectories~\cite{calinon2007learningrepresentinggeneralizing,ijspeert2013dynamicalmovementprimitives,paraschos2013probabilisticmovementprimitives,huang2018kernelizedmovementprimitives,khansarizadeh2011learningstabledynamicalsystems}.
More recently, Geometry-aware Policy Imitation treats demonstrations as geometric curves and derives distance fields that induce attraction and progression flows for policy synthesis~\cite{li2025geometryawarepolicyimitation}.
Another line of deep imitation learning focuses on multimodal continuous actions.
Implicit Behavioral Cloning uses energy-based objectives, and VQ-BeT uses latent action tokens for behavior generation~\cite{florence2021implicitbehavioralcloning,lee2024behaviorgenerationlatentactions}.
These works show that imitation data contains structure beyond a single regression target.
Our use of geometry focuses on the local variation of expert actions under similar observations as a training-time correction signal for one-step policies.
Compared with distance-field policy synthesis or action-distribution modeling, our objective uses local geometry during training while preserving single-evaluation deployment.

\section{Preliminaries}

\paragraph{From Multi-Step to One-Step Action Generation.}
In conditional imitation learning, given a dataset of expert demonstrations \(\mathcal{D} = \{(o_i, a^*_i)\}_{i=1}^N\), 
the objective is to learn a policy mapping an observation \(o \in \mathcal{O}\) to an expert action chunk \(a^*\) 
from the target data distribution \(p_{\text{data}}(a^* | o)\). 
Traditional behavior cloning (BC) frames this as a single-step direct mapping \(a = f_\theta(o)\), 
optimizing the isotropic mean squared error \(\mathbb{E}_{o, a^* \sim p_{\text{data}}} \|f_\theta(o) - a^*\|_2^2\). 
While deployment is highly efficient, this objective reduces generation to an isolated endpoint prediction, 
failing to exploit the geometric structure of the expert action manifold.
Generative action policies address this by constructing a continuous trajectory over a virtual time \(t \in [0, 1]\), 
where \(a_t\) denotes the intermediate state. Starting from an initial source sample \(a_0 \sim p_0\), 
diffusion models progressively reach the final action via a denoising path, 
typically optimized using the denoising loss \(\mathcal{L}_{\text{DP}} = \mathbb{E}[\|\epsilon_\theta(a_t, o, t) - \epsilon\|_2^2]\). 
Evolving from this, flow matching directly formulates the trajectory via an ordinary differential equation (ODE) 
\(da_t = v_\theta(a_t, o, t) dt\) driven by a velocity field, 
optimizing \(\mathcal{L}_{\text{FM}} = \mathbb{E}[\|v_\theta(a_t, o, t) - \frac{d}{dt}a_t\|_2^2]\). 
This multi-step evolution naturally retains the capacity for progressive error correction.

To reduce inference latency, recent acceleration algorithms compress the intermediate path to directly learn a one-step endpoint mapping or a large-stride velocity. 
For instance, Consistency Policy~\cite{prasad2024consistencypolicyacceleratedvisuomotor} learns a mapping \(f_\theta(o, a_t, t)\) by imposing a cross-time self-consistency constraint:
\(\mathcal{L}_{\text{CM}} = \mathbb{E}\left[ \left\| f_\theta(o, a_t, t) - f_{\phi}(o, a_{t'}, t') \right\|_2^2 \right]\),
enforcing any intermediate point on the trajectory to map to the same endpoint (where \(\phi\) is the target network parameter). 
Similarly, MeanFlow~\cite{geng2025meanflowsonestepgenerative} learns the average velocity \(u_\theta(a_t, o, r, t)\) over a time interval \([r, t]\) by satisfying an intrinsic identity between average and instantaneous velocities:
\(\mathcal{L}_{\text{MF}} = \mathbb{E}\left[ \left\| u_\theta(a_t, o, r, t) - \operatorname{sg}\left[ v_t - (t-r)\frac{d}{dt}u_\theta(a_t, o, r, t) \right] \right\|_2^2 \right]\),
where \(\operatorname{sg}[\cdot]\) denotes the stop-gradient operator.
However, whether relying on target-network distillation or higher-order derivative identities, the optimization processes encounter specific challenges. 
When the virtual intermediate evolution is bypassed, 
reconstructing the corrective signal for the one-step mapping without relying on expensive cross-step constraints remains a crucial problem.

\paragraph{Explicit Drifting in Behavior Cloning.}
Drifting models offer a mathematical framework to isolate the correction mechanism entirely within the training phase. 
Given the target data distribution \(p_{\text{data}}\) and generated distribution \(q_\theta\), 
Drifting constructs an explicit continuous vector field \(V_{p,q}(a) = V_p^+(a) - V_q^-(a)\) comprising attraction and repulsion. 
This field dictates the optimal displacement for a generated sample \(a\) toward the data manifold. Based on a kernel function \(k(\cdot, \cdot)\), 
the attraction and repulsion fields are respectively defined as:
\begin{equation}
    V_p^+(a) = \frac{\mathbb{E}_{y\sim p_{\text{data}}}[k(a, y)(y - a)]}{\mathbb{E}_{y\sim p_{\text{data}}}[k(a, y)]}, \qquad V_q^-(a) = \frac{\mathbb{E}_{y\sim q_\theta}[k(a, y)(y - a)]}{\mathbb{E}_{y\sim q_\theta}[k(a, y)]}.
\label{eq:drift_field}
\end{equation}
The one-step generator \(f_\theta\) is trained to regress this field-displaced frozen target:
\begin{equation}
    \mathcal{L}_{\mathrm{Drift}}(\theta) = \mathbb{E}_{a_0 \sim p_0} \left\| f_\theta(o, a_0, 0) - \operatorname{sg}\left[ f_\theta(o, a_0, 0) + V_{p,q}(f_\theta(o, a_0, 0)) \right] \right\|_2^2.
\label{eq:drift_loss}
\end{equation}
In unconditional generation, the rich sample sets supporting the distributions enable the vector field to provide smooth manifold guidance.
However, directly porting this explicit drifting mechanism to conditional behavior cloning encounters a fundamental structural degeneration. 
Consider the following proposition:

\begin{proposition}[Degeneration of Explicit Drifting in Behavior Cloning]
    \label{prop:degeneration_explicit_drifting}
    
    Assume that in behavior cloning, the empirical expert conditional contains only one recorded action \(a_i^*\) for a specific observation \(o_i\), 
    and the generator produces a single prediction \(a = f_\theta(o_i, a_0, 0)\) per training iteration. 
    The empirical conditional distributions degenerate into Dirac measures \(\hat{p}_{\text{data}}(y | o_i) = \delta(y - a_i^*)\) 
    and \(\hat{q}_\theta(y | o_i) = \delta(y - a)\). For any translation-invariant positive kernel \(k(\cdot, \cdot)\), 
    the empirical conditional drifting field mathematically evaluates exactly to \(V_{\hat{p},\hat{q}}(a | o_i) = a_i^* - a\). 
    Consequently, the field-displaced regression objective becomes exactly equivalent to the isotropic Mean Squared Error, 
    losing all local geometric guidance. (Proof. See Appendix~\ref{app:proof_prop1}.)
\end{proposition}

While the underlying control task may tolerate variations around \(a_i^*\) 
(e.g., alternative valid trajectories or recoverable states), 
the preceding analysis highlights that the empirical target distribution in standard BC training is too sparse to directly support the estimation of a continuous drifting field. 
To circumvent this degeneration in practice, models must resort to empirical approximations, 
such as synthesizing a localized empirical distribution by drawing multiple independent noise samples per observation during training and applying self-masking heuristics. 
While successful in avoiding the collapse to MSE, 
these approximations significantly increase computational overhead and entangle the correction field with transient sampling densities. 
This practical challenge motivates us to explore an alternative path: 
rather than regressing an explicit dynamic vector field, 
we extract the static local geometric structure directly from demonstrations.

\section{Implicit Drifting Policy}

Given the challenges of estimating explicit continuous vector fields under sparse empirical distributions, 
we explore an alternative path. Our core idea is to shift away from regressing dynamic drifting fields, 
and instead consider extracting the local geometric structure directly from static expert demonstrations, 
converting it into a potential objective that induces correction (see Fig~\ref{fig:pipeline}).

\begin{figure}[t]
\centering

\begin{overpic}[width=\linewidth]{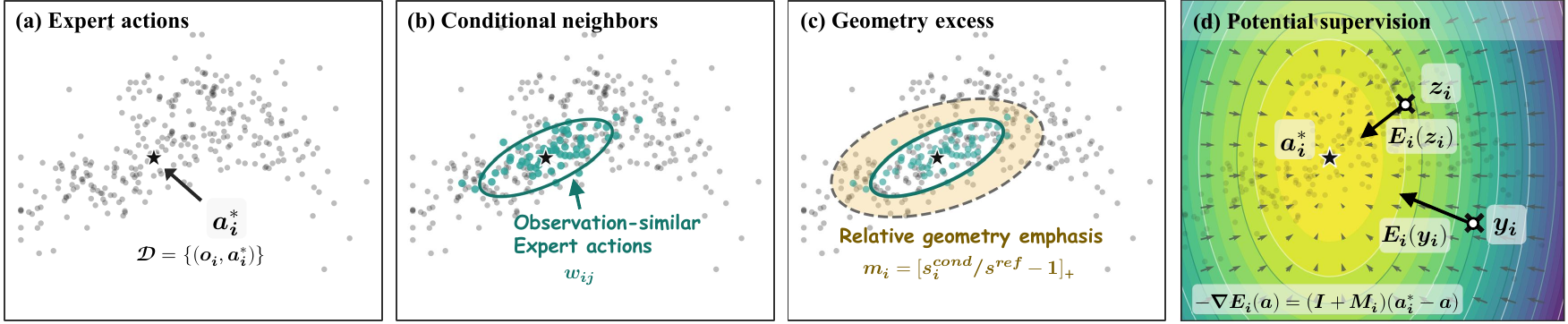}
\end{overpic}
\caption{
\textbf{Overview of Implicit Drifting Policy.}
(a) Expert action $a_i^*$ anchors the local action space.
(b) Weights $w_{ij}$ select observation-similar expert actions to form a Conditional Expert Geometry $G_i$ around $a_i^*$.
(c) Comparing $G_i$ with reference geometry $\Sigma_{\mathrm{ref}}$ yields the Local Geometry Excess $m_i$.
(d) The induced metric $M_i$ defines a potential $E_i$, where $-\nabla E_i(a)$ supervises both deployment proposal $y_i$ and expert-proximal evaluation $z_i$.
}
\vspace{-2ex}
\label{fig:pipeline}
\end{figure}

\subsection{Implicit Correction Objective and Local-Aware Evaluation}
\label{sec:4.1}
To avoid regressing a high-dimensional and unstable continuous drifting field, 
we consider constructing an implicit correction field directed toward the expert target \(a_i^*\). 
In the most basic case, a simple attraction field can be written as \(V_{\mathrm{base}}(a) = a_i^* - a\). 
However, this simple attraction is isotropic. In multi-step generative models, 
isotropic targets, like MSE loss, are sufficient because the model can implicitly adjust and align with the action manifold through hundreds of inference iterations along the intermediate trajectory. 
However, for a one-step generator, it loses the ability to progressively correct errors along the intermediate path. 
If the training objective does not explicitly reinforce the tightness of constraints across different action dimensions, 
the prediction is highly prone to diverging from the valid manifold structure in a single stride.
Therefore, for a given predicted action \(a\), we consider applying a geometry-aware correction vector:
\begin{equation}
    \Delta_i(a) = M_i(a_i^* - a),
\label{eq:geo_correction}
\end{equation}
where \(M_i \succeq 0\) is a positive semi-definite symmetric matrix acting as geometry-adaptive weights, 
determining the strength of the correction force along different directions. 

For fixed symmetric $M_i\succeq0$, the geometric correction is the negative gradient of the following potential (Appendix~\ref{app:proof_prop2}):
\begin{equation}
        E_i^{\text{geo}}(a) = \frac{1}{2} (a - a_i^*)^\top M_i (a - a_i^*).
    \label{eq:geo_potential}
\end{equation}
The complete geometric potential objective incorporates the isotropic error of behavior cloning:
\begin{equation}
    E_i(a) = \frac{1}{2} \| a - a_i^* \|_2^2 + E_i^{\text{geo}}(a) = \frac{1}{2}[\underbrace{ \| a - a_i^* \|_2^2}_{\text{isotropic error}} + \underbrace{(a - a_i^*)^\top M_i (a - a_i^*)}_{\text{geometry-adaptive penalty}}]
\label{eq:full_potential}
\end{equation}
whose full negative gradient naturally yields a combined correction force: \(-\nabla_a E_i(a) = (I + M_i)(a_i^* - a)\) (Fig~\ref{fig:pipeline} (d)).
The first term of this potential is equivalent to the isotropic error of behavior cloning, 
while the second term adaptively increases the penalty for constrained directions based on \(M_i\). 
We transform the regression of an explicit vector field into the minimization of a potential. 

However, the constraint characterized by matrix \(M_i\) is statistically grounded primarily within the local vicinity of the expert action \(a_i^*\). 
If the optimization process applies this potential constraint exclusively at the standard initial prediction point \(y_i = f_\theta(o_i, a_0, 0)\), 
the network might struggle to adequately learn the true topology of the local manifold. To ensure that this local potential effectively guides network learning, 
we consider introducing an expert-proximal evaluation anchor during training. 
By selecting a time scalar \(t_* \in (0, 1)\) close to \(1\), 
we construct a probing input \(\tilde{a}_i = a_i^* + (1 - t_*)\epsilon_i\) (where \(\epsilon_i \sim \mathcal{N}(0, I)\)) by injecting exploratory noise proportional to the distance around the expert action. 
We evaluate the mapping at this proximal location, yielding \(z_i = f_\theta(o_i, \tilde{a}_i, t_*)\).
Here, $t$ indexes the virtual generation stage of the action input. We use $a_0=0$ for the direct prediction.
Therefore, the single-step prediction and the local probing are controlled by the same geometric potential objective, 
formulating our total training objective as:
\begin{equation}
    \mathcal{L} = \mathbb{E}_{i \sim \mathcal{D}} \left[ E_i(y_i) + \lambda_{\mathrm{prox}}\, E_i(z_i) \right],
    \label{eq:idp_total_loss}
\end{equation}
where \(\lambda_{\mathrm{prox}}\) is a hyperparameter balancing the standard endpoint regression and the expert-proximal evaluation. Implementation settings are given in Appendix~\ref{app:hyper}. The expectation also covers the input perturbation noise.
This construction establishes the skeleton of the entire implicit correction framework. 
The core challenge then becomes how to extract the adaptive constraint matrix \(M_i\) from static demonstration data to reflect the current control state.

\subsection{Conditional Expert Geometry as the Correction Source}

To ensure the adaptive penalty effectively limits invalid actions, 
we consider grounding the construction of matrix \(M_i\) in the authentic structure of the demonstration data. 
In control manifolds, similar observation states usually correspond to analogous robot poses or environmental layouts, 
which physically often impose similar operational bottlenecks and obstacle avoidance constraints. 
Under similar observations, expert-action variation provides a local reference for which action coordinates are more consistent. Comparing this variation with the minibatch reference lets us strengthen the potential along relatively more consistent coordinates while retaining the baseline penalty elsewhere.

To extract this structure, we need to identify and aggregate similar control contexts in the feature space. 
We normalize the embeddings to remove feature-magnitude effects and apply stop-gradient to keep the geometry statistics fixed during backpropagation:
\begin{equation}
    h_i = \frac{\operatorname{sg}[\phi_\theta(o_i)]}{\|\operatorname{sg}[\phi_\theta(o_i)]\|_2 + \varepsilon}.
    \label{eq:obs_features}
\end{equation}

To evaluate the correlation between different states, 
we first compute the pairwise feature inner products \(s_{ij} = h_i^\top h_j\). 
The scale of similarity differences can vary across rows, affecting how concentrated the softmax weights become.
To mitigate this issue, we adopt row-normalization \(\bar{s}_{ij} = \frac{s_{ij} - \mu_i}{\sigma_i + \varepsilon}\), 
and compute the neighborhood weights \(w_{ij}\propto\exp(\bar s_{ij})\) by applying softmax over all $B$ samples in the minibatch.
Based on these weights, 
to quantify the degree of action divergence under the current observation, 
we define the Conditional Expert Geometry matrix within the local space of the expert action \(a_i^*\) (Fig~\ref{fig:pipeline} (b)):
\begin{equation}
    G_i = \sum_{j=1}^{B} w_{ij}(a_j^*-a_i^*)(a_j^*-a_i^*)^\top.
    \label{eq:cond_expert_geometry}
\end{equation}
$G_i$ describes the spread of related expert actions around $a_i^*$. Smaller spread along a direction indicates more consistent expert actions around the target. This local geometry provides the basis for constructing the correction weights $M_i$ (Appendix~\ref{app:kernel_smoothing}).

\subsection{Local Geometry Excess}

Directly utilizing the inverse of \(G_i\) as \(M_i\) could introduce biases. 
This is because certain action coordinates may naturally maintain extremely low variance globally across the entire task. 
We consider it necessary to distinguish the condition-specific constraints triggered purely by the current observation from this inherent global action scale.

In a general multivariate form, distinguishing the condition-specific constraints from the global marginal prior typically requires comparing their inverse covariance matrices. Let \(P_i^{\mathrm{cond}} = (G_i + \varepsilon I)^{-1}\) and \(P^{\mathrm{ref}} = (\Sigma_{\text{ref}} + \varepsilon I)^{-1}\). The geometry excess can be conceptualized via a non-negative truncated comparison, such as \(M_i \propto [(P^{\mathrm{ref}})^{-1/2} P_i^{\mathrm{cond}} (P^{\mathrm{ref}})^{-1/2} - I]_+\), which usually involves expensive generalized eigenvalue decomposition operations. Considering computational efficiency and to ensure relatively independent penalties across dimensions, we adopt a diagonal instantiation in practice. 
Specifically, we compute the diagonal conditional variance \(v_{i,d}^{\mathrm{cond}} = \sum_{j=1}^{B} w_{ij}(a_{j,d}^*-a_{i,d}^*)^2\) 
and the global reference variance \(v_d^{\mathrm{ref}} = \frac1B\sum_{i=1}^{B}(a_{i,d}^*-\bar a_d^*)^2\).
Both statistics are computed within the current minibatch: the local second moment is centered at $a_i^*$, and the reference variance is centered at the batch mean. 
To eliminate dimensional impacts and compare the tightness of different dimensions on the same relative baseline, 
we convert them into normalized inverse-standard-deviation scales:
\begin{equation}
    s_{i,d}^{\mathrm{cond}}
    =
    \frac{\big( v_{i,d}^{\mathrm{cond}} + \varepsilon \big)^{-1/2}}
    {\mathbb{E}_{k}\big[ (v_{i,k}^{\mathrm{cond}} + \varepsilon)^{-1/2} \big]},
    \qquad
    s_d^{\mathrm{ref}}
    =
    \frac{\big( v_d^{\mathrm{ref}} + \varepsilon \big)^{-1/2}}
    {\mathbb{E}_{k}\big[ (v_k^{\mathrm{ref}} + \varepsilon)^{-1/2} \big]}.
    \label{eq:normalized_scales}
\end{equation}

Since we only wish to penalize dimensions that become more strictly constrained due to the current observation than they usually are, 
we consider utilizing the filtering property of the one-sided ReLU operator to define the Geometry Excess (Fig~\ref{fig:pipeline} (c)):

\begin{equation}
 m_{i,d}=\operatorname{ReLU}\left[   
 \frac{s_{i,d}^{\mathrm{cond}}}{s_d^{\mathrm{ref}}+\varepsilon}-1
  \right],
    \qquad
    M_i=\operatorname{Diag}(m_i).
\label{eq:geometry_excess}
\end{equation}

The ReLU retains only the part of the normalized local-to-reference ratio exceeding one, leaving a baseline coefficient of one in the other coordinates. For action chunks, $m_{i,\ell,d}$ is retained separately for each time--action coordinate pair.
Consequently, we filter out inherent features that do not surpass the baseline, 
extract the specific geometry associated with the current state, 
and use it as the adaptive weight matrix \(M_i\) for the potential in Section~\ref{sec:4.1}.

\section{Experiments}
\begin{table}[t]
    \centering
    \caption{\textbf{Performance on 3D pointcloud manipulation.} We assess performance on 56 challenging tasks with 3 random seeds.
The aggregate standard deviation averages the per-task standard deviations across seeds.}
    \label{tab:3d_summary}
    \setlength{\tabcolsep}{4.5pt}
    \resizebox{\textwidth}{!}{%
    \begin{tabular}{l c ccc cccc cccc c}
        \toprule
        \multirow{2.5}{*}{\textbf{Methods}} & \multirow{2.5}{*}{\textbf{NFE}} & \multicolumn{3}{c}{\textbf{Adroit}} & \multicolumn{4}{c}{\textbf{DexArt}} & \multicolumn{4}{c}{\textbf{MetaWorld}} & \multirow{2.5}{*}{\textbf{Avg.}} \\
        \cmidrule(lr){3-5} \cmidrule(lr){6-9} \cmidrule(lr){10-13}
        & & Hammer & Door & Pen & Laptop & Faucet & Bucket & Toilet & Easy (28) & Medium (11) & Hard (5) & Very Hard (5) & \\
        \midrule
        \multicolumn{14}{l}{\textbf{\textit{Multi-Step:}}} \\
        \textcolor{gray}{DP3$\dagger$} & \textcolor{gray}{10} & \textcolor{gray}{\textbf{100}}\std{0} & \textcolor{gray}{62}\std{4} & \textcolor{gray}{43}\std{6} & \textcolor{gray}{83}\std{1} & \textcolor{gray}{63}\std{2} & \textcolor{gray}{46}\std{2} & \textcolor{gray}{82}\std{4} & \textcolor{gray}{90.9}\std{1.4} & \textcolor{gray}{61.6}\std{6.5} & \textcolor{gray}{38.0}\std{4.2} & \textcolor{gray}{49.0}\std{6.8} & \textcolor{gray}{73.9}\std{3.3} \\
        \textcolor{gray}{Flow} & \textcolor{gray}{10} & \textcolor{gray}{\textbf{100}}\std{0} & \textcolor{gray}{72}\std{3} & \textcolor{gray}{69}\std{3} & \textcolor{gray}{\textbf{92}}\std{3} & \textcolor{gray}{37}\std{5} & \textcolor{gray}{34}\std{5} & \textcolor{gray}{80}\std{7} & \textcolor{gray}{66.7}\std{7.7} & \textcolor{gray}{41.3}\std{7.4} & \textcolor{gray}{40.3}\std{7.2} & \textcolor{gray}{46.7}\std{7.6} & \textcolor{gray}{57.9}\std{7.0} \\ \cdashlinelr{1-14}
        DP3$^*$ & 10 & \textbf{100}\std{0} & 58\std{5} & 53\std{3} & 83\std{2} & 40\std{2} & 25\std{1} & 74\std{2} & 91.5\std{0.9} & 74.5\std{5.7} & 54.2\std{1.8} & 71.8\std{4.0} & 79.4\std{2.4} \\
        Flow$^*$ & 10 & 99\std{1} & 69\std{3} & 45\std{3} & 85\std{0} & 42\std{4} & 26\std{3} & 74\std{3} & 92.2\std{0.7} & 74.6\std{5.0} & 55.2\std{3.0} & 70.8\std{3.0} & 79.9\std{2.2} \\
        \midrule
        \multicolumn{14}{l}{\textbf{\textit{Flow-Map / Shortcut / Few-Step:}}} \\
        \textcolor{gray}{CP} & \textcolor{gray}1 & \textcolor{gray}{98}\std{3} & \textcolor{gray}{67}\std{5} & \textcolor{gray}{65}\std{5} & \textcolor{gray}{87}\std{3} & \textcolor{gray}{44}\std{4} & \textcolor{gray}{29}\std{4} & \textcolor{gray}{82}\std{4} & \textcolor{gray}{68.3}\std{1.8} & \textcolor{gray}{32.1}\std{4.3} & \textcolor{gray}{28.7}\std{5.2} & \textcolor{gray}{36.0}\std{2.3} & \textcolor{gray}{54.7}\std{2.9} \\
        \textcolor{gray}{MP1} & \textcolor{gray}1 & \textcolor{gray}{\textbf{100}}\std{0} & \textcolor{gray}{74}\std{3} & \textcolor{gray}{72}\std{3} & \textcolor{gray}{88}\std{5} & \textcolor{gray}{42}\std{1} & \textcolor{gray}{33}\std{0} & \textcolor{gray}{82}\std{2} & \textcolor{gray}{70.5}\std{6.1} & \textcolor{gray}{37.9}\std{5.9} & \textcolor{gray}{34.3}\std{9.5} & \textcolor{gray}{32.3}\std{5.5} & \textcolor{gray}{57.4}\std{5.8} \\
        \textcolor{gray}{OneDP} & \textcolor{gray}1 & \textcolor{gray}{99}\std{1} & \textcolor{gray}{67}\std{2} & \textcolor{gray}{65}\std{2} & \textcolor{gray}{89}\std{4} & \textcolor{gray}{45}\std{3} & \textcolor{gray}{33}\std{6} & \textcolor{gray}{\textbf{83}}\std{4} & \textcolor{gray}{77.7}\std{0.5} & \textcolor{gray}{39.6}\std{9.2} & \textcolor{gray}{38.7}\std{4.4} & \textcolor{gray}{41.7}\std{1.6} & \textcolor{gray}{62.4}\std{3.0} \\
        \textcolor{gray}{OFP} & \textcolor{gray}1 & \textcolor{gray}{\textbf{100}}\std{0} & \textcolor{gray}{\textbf{79}}\std{7} & \textcolor{gray}{\textbf{76}}\std{7} & \textcolor{gray}{{92}}\std{3} & \textcolor{gray}{{46}}\std{4} & \textcolor{gray}{\textbf{39}}\std{0} & \textcolor{gray}{81}\std{3} & \textcolor{gray}{87.9}\std{4.6} & \textcolor{gray}{52.4}\std{5.7} & \textcolor{gray}{43.3}\std{2.4} & \textcolor{gray}{49.2}\std{0.7} & \textcolor{gray}{71.6}\std{4.1} \\
        \cdashlinelr{1-14}
        CP$^*$ & 1 & \text{97}\std{2} & \text{66}\std{0} & \text{54}\std{2} & \text{76}\std{11} & \text{38}\std{0} & \text{25}\std{5} & \text{73}\std{4} & \text{79.8}\std{14} & \text{59.4}\std{12} & \text{41.4}\std{14} & \text{62.0}\std{8.2} & \text{68.4}\std{11.9} \\
        MP1$^*$ & 1 & \text{99}\std{1} & 48\std{2} & 58\std{4} & 82\std{3} & {46}\std{2} & 26\std{3} & 69\std{3} & 91.3\std{1.2} & \textbf{75.4}\std{5.6} & \textbf{57.6}\std{4.6} & 72.8\std{4.2} & 79.7\std{2.8} \\
        FlowPolicy$^*$ & 1 & 98\std{2} & 66\std{1} & 55\std{2} & 86\std{2} & 40\std{1} & 28\std{0} & 76\std{3} & {90.4}\std{0.8} & {67.0}\std{6.2} & {51.2}\std{4.8} & {67.0}\std{3.0} & {76.9}\std{2.5} \\
        \midrule
        \multicolumn{14}{l}{\textbf{\textit{Drifting Methods:}}} \\
        Naive Drifting$^*$ & 1 & \textbf{100}\std{0} & 67\std{3} & 43\std{5} & {95}\std{4} & {46}\std{3} & 29\std{1} & 82\std{6} & {92.7}\std{1.8} & 70.7\std{6.5} & 54.2\std{5.8} & {73.2}\std{5.2} & {79.8}\std{3.5} \\
        Ours$^*$ & 1 & \textbf{100}\std{0} & 67\std{2} & 52\std{4} & \textbf{96}\std{3} & \textbf{48}\std{1} & 30\std{0} & \textbf{83}\std{6} & \textbf{93.2}\std{0.8} & 73.6\std{3.9} & 56.4\std{4.2} & \textbf{74.2}\std{1.4} & \textbf{81.3}\std{2.0} \\
        \bottomrule
    \end{tabular}%
    }
    \begin{flushleft}
        \scriptsize $^*$ Results are reproduced on a single NVIDIA A40. $^\dagger$ Results are cited from DP3. All other baselines cited from OFP.
    \end{flushleft}
\end{table}

We design our experiments to examine four questions that follow the construction of IDP:

\begin{itemize}
    \item \textbf{H1: Performance.}
    Can IDP improve one-step action generation across simulated and real manipulation tasks, while narrowing the performance gap to multi-step generative policies?

    \item \textbf{H2: Implicit Correction and Local-Aware Evaluation.}
    Can IDP benefit from local geometric supervision by being evaluated near expert actions during training?

    \item \textbf{H3: Conditional Expert Geometry.}
    Do expert actions under similar observations provide useful local geometry for guiding action correction?

    \item \textbf{H4: Local Geometry Excess.}
    Does comparing conditional geometry with reference geometry isolate condition-specific constraints from marginal action variation?
\end{itemize}





\subsection{Experimental Setup}

\textbf{Tasks.} We evaluate IDP on three settings: (i) 2D state- and image-based benchmarks following the MIP protocol~\cite{pan2026adonoisingdispellingmyths} on Robomimic (Lift, Can, Square, Transport, with both \texttt{ph} and \texttt{mh}), Tool-Hang, and PushT, (ii) 3D pointcloud-conditioned manipulation following OFP~\cite{li2026onestepflowpolicyselfdistillation} and FlowPolicy~\cite{zhang2024flowpolicyenablingfastrobust} on 56 tasks across Adroit~\cite{rajeswaran2017learning}, DexArt~\cite{bao2023dexart}, and MetaWorld~\cite{yu2020meta}, and (iii) a real-world ``Pick Peach'' task on a single-arm ALOHA platform with a strictly bimodal demonstration distribution.

\textbf{Baselines.} We compare against multi-step generative policies (DP~\cite{chi2024diffusionpolicyvisuomotorpolicy}, DP3~\cite{ze20243d}, Flow~\cite{lipman2023flowmatchinggenerativemodeling}), few-step / one-step accelerated policies (CP~\cite{prasad2024consistencypolicyacceleratedvisuomotor}, MIP~\cite{pan2026adonoisingdispellingmyths}, MP1~\cite{sheng2025mp1meanflowtamespolicy}, OneDP~\cite{wang2024onestepdiffusionpolicyfast}, OFP~\cite{li2026onestepflowpolicyselfdistillation}, FlowPolicy~\cite{zhang2024flowpolicyenablingfastrobust}), and the explicit drifting baseline (Naive Drifting~\cite{deng2026generativemodelingdrifting}).

\textbf{Metrics.} We report success rates averaged across 2 architectures (Transformer / CNN) and 3 seeds for 2D tasks, and across 3 seeds for 3D tasks. Detailed task descriptions, baseline reproduction settings, hyperparameters, and compute resources are deferred to Appendix~\ref{app:setup}--\ref{app:impl}.



\begin{table}[t]
    \centering
    \caption{Performance on 2D \textbf{state-based} robotics manipulation tasks. We report the best success rate / the mean over the last five evaluations during training. Each result is averaged over two network architectures (Transformer and CNN), three training seeds, and 50 environment initial conditions.}
    \label{tab:2d_summary_state}
    \setlength{\tabcolsep}{4.5pt}
    \resizebox{\textwidth}{!}{%
    \begin{tabular}{lcccccccccccc}
        \toprule
        \multirow{2.5}{*}{\textbf{Methods}} & \multirow{2.5}{*}{\textbf{NFE}} & \multicolumn{2}{c}{Lift} & \multicolumn{2}{c}{Can} & \multicolumn{2}{c}{Square} & \multicolumn{2}{c}{Transport} & \multirow{2.5}{*}{Tool-Hang} & \multirow{2.5}{*}{PushT} & \multirow{2.5}{*}{\textbf{Avg.}} \\ \cmidrule(lr){3-4} \cmidrule(lr){5-6} \cmidrule(lr){7-8} \cmidrule(lr){9-10}
             &     & mh          & ph      & mh        & ph        & mh        & ph        & mh        & ph        &           &           &           \\ \midrule
        \multicolumn{13}{l}{\textbf{\textit{Multi-Step:}}} \\
        DP                       & 100                  & \textbf{1.00}/\textbf{0.99} & \textbf{1.00}/\textbf{1.00} & \textbf{0.98}/0.92 & \textbf{1.00}/\textbf{0.98} & 0.81/\textbf{0.72} & 0.96/\textbf{0.92} & \textbf{0.45}/0.34 & 0.73/\textbf{0.62} & \textbf{0.84}/\textbf{0.71} & 0.95/0.92 & \textbf{0.87}/\textbf{0.81} \\
        Flow                     & 9                    & \textbf{1.00}/\textbf{0.99} & \textbf{1.00}/\textbf{1.00} & 0.96/0.92 & \textbf{1.00}/\textbf{0.98} & 0.76/0.65 & 0.95/0.90 & 0.33/0.23 & \textbf{0.74}/0.59 & 0.70/0.57 & 0.96/\textbf{0.94} & 0.84/0.78 \\ \midrule
        \multicolumn{13}{l}{\textbf{\textit{Few-Step / Flow-Map / Shortcut:}}} \\
        MIP                      & 2                    & \textbf{1.00}/\textbf{0.99} & \textbf{1.00}/\textbf{1.00} & 0.96/\textbf{0.94} & \textbf{1.00}/\textbf{0.98} & 0.80/0.69 & 0.96/0.88 & 0.44/\textbf{0.35} & 0.71/0.58 & 0.68/0.55 & 0.96/\textbf{0.94} & 0.85/0.79 \\
        CP                       & 1                    & 0.99/0.97 & \textbf{1.00}/\textbf{1.00} & 0.83/0.74 & 0.99/0.96 & 0.49/0.45 & 0.88/0.82 & 0.01/0.01 & 0.32/0.24 & 0.43/0.33 & 0.94/0.91 & 0.69/0.64 \\
        MP1                      & 1                    & \textbf{1.00}/\textbf{0.99} & \textbf{1.00}/\textbf{1.00} & 0.96/0.87 & \textbf{1.00}/0.97 & 0.77/0.62 & 0.96/0.88 & 0.24/0.14 & 0.52/0.39 & 0.76/0.63 & 0.95/0.92 & 0.82/0.74 \\ \midrule
        \multicolumn{13}{l}{\textbf{\textit{Drifting Methods:}}} \\
        Naive Drifting           & 1                    & \textbf{1.00}/\textbf{0.99} & \textbf{1.00}/\textbf{1.00} & 0.96/0.91 & 0.99/0.96 & 0.70/0.60 & \textbf{0.97}/0.90 & 0.30/0.22 & 0.68/0.58 & 0.53/0.42 & 0.95/0.91 & 0.81/0.75 \\
        Ours                     & 1                    & \textbf{1.00}/\textbf{0.99} & \textbf{1.00}/\textbf{1.00} & \textbf{0.98}/\textbf{0.94} & \textbf{1.00}/0.97 & \textbf{0.82}/0.68 & 0.95/0.88 & 0.41/0.29 & 0.72/0.58 & 0.81/0.60 & \textbf{0.97}/0.92 & \textbf{0.87}/0.79 \\ \bottomrule
    \end{tabular}
    
    }
    \vspace{-1ex}
\end{table}

\begin{table}[t]
    \centering
    \caption{Performance on 2D \textbf{image-based} robotics manipulation tasks. We report the best success rate / the mean over the last five evaluations during training. Each result is averaged over two network architectures (Transformer and CNN), three training seeds, and 50 environment initial conditions.}
    \label{tab:2d_summary_image}
    \setlength{\tabcolsep}{4.5pt}
    \resizebox{\textwidth}{!}{%
    \begin{tabular}{lcccccccccc}
        \toprule
        \multirow{2.5}{*}{\textbf{Methods}} & \multirow{2.5}{*}{\textbf{NFE}} & \multicolumn{2}{c}{Lift} & \multicolumn{2}{c}{Can} & \multicolumn{2}{c}{Square} & \multirow{2.5}{*}{Tool-Hang} & \multirow{2.5}{*}{PushT} & \multirow{2.5}{*}{\textbf{Avg.}} \\ \cmidrule(lr){3-4} \cmidrule(lr){5-6} \cmidrule(lr){7-8}
             &     & mh          & ph      & mh        & ph        & mh        & ph        &           &           &           \\ \midrule
        \multicolumn{11}{l}{\textbf{\textit{Multi-Step:}}} \\
        DP                       & 100                  & \textbf{1.00}/\textbf{0.98} & \textbf{1.00}/0.99 & 0.90/0.82 & 0.98/\textbf{0.94} & \textbf{0.87}/0.76 & 0.97/\textbf{0.91} & \textbf{0.48}/0.33 & \textbf{0.93}/\textbf{0.88} & \textbf{0.89}/\textbf{0.83} \\
        Flow                     & 9                    & \textbf{1.00}/\textbf{0.98} & \textbf{1.00}/0.99 & 0.93/0.83 & \textbf{0.99}/\textbf{0.94} & 0.86/\textbf{0.78} & \textbf{0.98}/0.90 & 0.47/\textbf{0.35} & 0.92/0.86 & \textbf{0.89}/\textbf{0.83} \\ \midrule
        \multicolumn{11}{l}{\textbf{\textit{Few-Step / Flow-Map / Shortcut:}}} \\
        MIP                      & 2                    & \textbf{1.00}/0.96 & \textbf{1.00}/0.99 & 0.89/0.83 & 0.97/0.93 & 0.82/0.71 & 0.95/0.83 & 0.46/0.28 & 0.87/0.84 & 0.87/0.80 \\
        CP                       & 1                    & \textbf{1.00}/\textbf{0.98} & \textbf{1.00}/\textbf{1.00} & 0.77/0.63 & 0.80/0.60 & 0.70/0.59 & 0.92/0.86 & 0.22/0.12 & 0.88/0.83 & 0.79/0.70 \\
        MP1                      & 1                    & \textbf{1.00}/0.97 & \textbf{1.00}/0.99 & 0.87/0.77 & 0.77/0.68 & 0.85/\textbf{0.78} & 0.96/0.90 & 0.43/0.27 & 0.86/0.79 & 0.84/0.77 \\ \midrule
        \multicolumn{11}{l}{\textbf{\textit{Drifting Methods:}}} \\
        Naive Drifting           & 1                    & 0.99/0.96 & \textbf{1.00}/0.98 & 0.90/0.79 & 0.95/0.84 & 0.80/0.69 & 0.92/0.84 & 0.46/0.26 & 0.91/\textbf{0.86} & 0.87/0.78 \\
        Ours                     & 1                    & \textbf{1.00}/0.97 & \textbf{1.00}/\textbf{0.99} & \textbf{0.95}/\textbf{0.88} & \textbf{0.99}/\textbf{0.94} & 0.82/0.74 & 0.94/0.86 & \textbf{0.48}/0.33 & \textbf{0.92}/\textbf{0.86} & \textbf{0.89}/0.82 \\ \bottomrule
    \end{tabular}
    }
    \vspace{-1ex}
\end{table}

\subsection{One-Step Policy Performance (H1)}

\textbf{Simulated Manipulation.} To address \textbf{H1}, we evaluate the control performance of IDP across the 2D and 3D simulated benchmarks. As detailed in Tables \ref{tab:2d_summary_state}, \ref{tab:2d_summary_image}, and \ref{tab:3d_summary}, IDP achieves highly competitive success rates across state, image, and pointcloud modalities. Compared to the explicit Naive Drifting baseline, IDP generally yields more stable and improved performance, suggesting that the implicit geometric objective can serve as a more reliable correction mechanism than empirical vector field estimation. When assessed against other leading single-step methods (e.g., CP, MP1, and OFP), IDP demonstrates solid performance, matching or slightly exceeding their success rates on the majority of tasks. Furthermore, operating purely at $\text{NFE}=1$, IDP performs closely to the multi-step generative references (e.g., 100-step DP and 10-step DP3) in several settings. These results indicate that incorporating implicit geometry effectively aids in synthesizing high-quality actions without necessarily relying on iterative inference steps.

\begin{figure}[t]
    \centering
    \begin{minipage}[t]{0.30\linewidth}\vspace{0pt}
        \centering

        \includegraphics[width=0.8\linewidth]{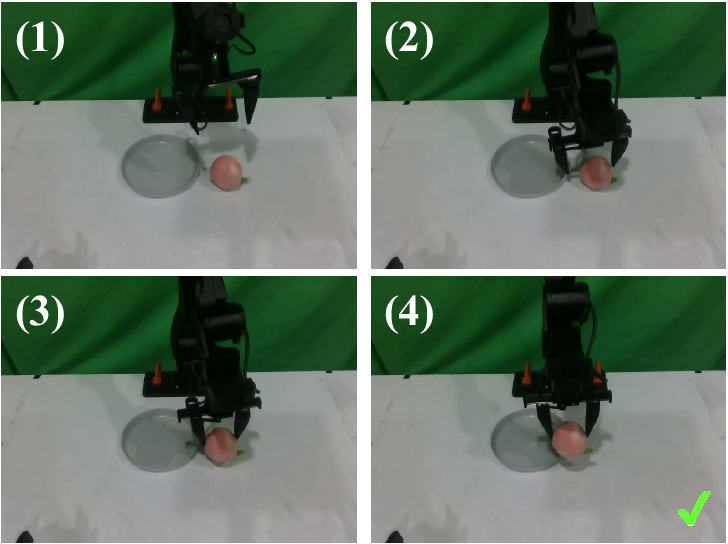}
        \captionof{figure}{Real-world ``Pick Peach'' task visualization.}
\label{fig:real_pickpeach}
    \end{minipage}\hfill%
    \begin{minipage}[t]{0.30\linewidth}\vspace{0pt}
        \centering
        \captionof{table}{Success rates on the real-world ``Pick Peach'' task.}
        \label{tab:real_robot}
        
        \resizebox{\linewidth}{!}{%
        \begin{tabular}{llc}
            \toprule
            \textbf{Modality} & \textbf{Method} & \textbf{SR (\%)} \\
            \midrule
            \multirow{4}{*}{\begin{tabular}[c]{@{}l@{}}\textbf{State-based} \\ \textbf{(\texttt{qpos})}\end{tabular}} 
            & DP & 10 \\
            & MIP & \textbf{40} \\
            & Naive Drifting & 0 \\
            & Ours & 30 \\
            \midrule
            \multirow{4}{*}{\begin{tabular}[c]{@{}l@{}}\textbf{Vision-based} \\ \textbf{(\texttt{image})}\end{tabular}} 
            & DP & 0 \\
            & MIP & 0 \\
            & Naive Drifting & 0 \\
            & Ours & \textbf{50} \\
            \bottomrule
        \end{tabular}
        }
    \end{minipage}\hfill%
    \begin{minipage}[t]{0.35\linewidth}\vspace{0pt}
        \centering
        \captionof{table}{Ablation results on ToolHang-State (Max / Avg).}
        \label{tab:ablation}
        \resizebox{\linewidth}{!}{%
        \begin{tabular}{lcc}
            \toprule
            \textbf{Method} & \textbf{ToolHang-State} \\
            \midrule
            Naive Drifting & 53.2 / 42.4 \\
            IDP (Ours) & \textbf{81.0} / \textbf{60.0} \\
            \quad w/o terminal eval & 40.0 / 25.5 \\
            \quad uniform neighbors & 65.0 / 59.5 \\
            \quad shuffled neighbors & 75.0 / 52.0 \\
            \quad w/o reference geometry & 78.0 / 54.5 \\
            \quad absolute excess & 75.0 / 57.0 \\
            \bottomrule
        \end{tabular}
        }
    \end{minipage}
    \vspace{-1ex}
\end{figure}




\begin{figure}[t]
\centering
\includegraphics[width=\linewidth]{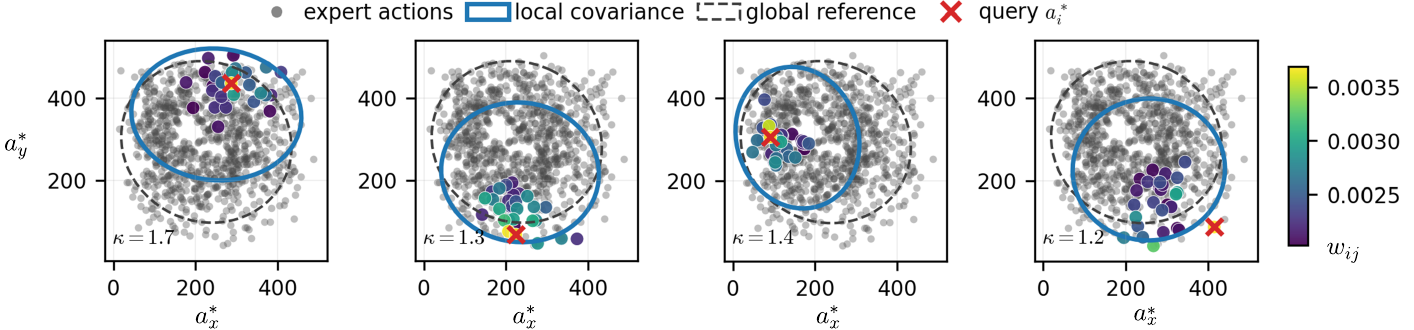}
\caption{Observation-conditioned expert-action scatter on PushT-State. Colored points denote weighted neighbors. Blue and gray ellipses illustrate local and global action scatter.}
\vspace{-1ex}
\label{fig:cov_pusht}
\end{figure}

\textbf{Real-World Manipulation.} We further validate IDP on a single-arm ALOHA platform on a ``Pick Peach'' task with a strictly bimodal demonstration distribution ($50\%$ left, $50\%$ right). Each method is evaluated over $10$ trials ($5$ per mode) under both state-based (\texttt{qpos}) and vision-based modalities. The full hardware and protocol are deferred to Appendix~\ref{app:setup_real}.
In the vision-based setting (Table~\ref{tab:real_robot}), IDP achieves $50\%$ success while DP, Naive Drifting, and MIP all fail at $0\%$. Qualitatively, the failing baselines exhibit a consistent \textit{deviation from the action manifold}: they grasp at locations noticeably \textit{behind} the target peach, and multi-step DP additionally outputs invalid intermediate gripper commands (e.g., freezing at $0.769$) while breaking the high-frequency control loop with its sluggish inference. IDP adheres to valid grasping poses with real-time reactivity, supporting our geometric hypothesis. In the state-based setting, where joint-angle proximity is a poor proxy for task context, IDP-\texttt{qpos} ($30\%$) underperforms MIP-\texttt{qpos} ($40\%$). We discuss this regime in Appendix~\ref{app:limitations}.

All evaluated methods, including IDP, MIP, and 100-step DP, learn a single placement mode under perfectly balanced bimodal demonstrations. Multi-step DP also fails to complete the task in this setting. IDP retains action-validity within the learned mode. Mode coverage remains an open problem we discuss in Appendix~\ref{app:limitations}.

\subsection{Ablation Study (H2 $\sim$ H4)}
\label{sec:ablation}

We ablate the core designs of IDP to address hypotheses \textbf{H2} $\sim$ \textbf{H4}. As our geometric objective primarily benefits tasks with narrow successful action regions, we conduct quantitative ablations on ToolHang-State (Table~\ref{tab:ablation}). 

\textbf{Expert-proximal evaluation (H2):}
Removing the expert-proximal probe collapses the success rates on ToolHang-State from $81.0\%/60.0\%$ to $40.0\%/25.5\%$, well below Naive Drifting ($53.2\%/42.4\%$). The probe is therefore essential for internalizing the local correction structure into the one-step generator.
\textbf{Conditional expert geometry (H3):}
Replacing observation-conditioned neighborhood weights with uniform or shuffled ones degrades the success rates to $65.0\%/59.5\%$ and $75.0\%/52.0\%$ respectively, confirming that the geometry must be condition-specific. Fig.~\ref{fig:cov_pusht} visualizes this construction on PushT-State: observation-similar demonstrations form local covariances that deviate from the global reference with moderate anisotropy ($\kappa\!\approx\!1.2\text{--}1.7$). Further component ablations are reported in Appendix~\ref{app:matched}.
\textbf{Reference geometry and local excess (H4):}
Removing the reference geometry reduces success from 81.0\%/60.0\% to 78.0\%/54.5\%, while absolute excess gives 75.0\%/57.0\%. The full construction outperforms both variants, supporting the reference comparison and positive relative excess for geometry-aware supervision.

\section{Conclusion}

This paper addresses the loss of intermediate error correction in one-step action generation by proposing the Implicit Drifting Policy (IDP). IDP extracts local geometric constraints directly from expert actions under similar observations. By combining this conditional expert geometry with an expert-proximal training evaluation, IDP structurally internalizes the local manifold topology into a single-step generator. Extensive evaluations across 2D, 3D, and real-world manipulation benchmarks demonstrate that IDP effectively adheres to valid action manifolds, substantially improving upon explicit drifting methods while achieving superior geometric precision against strong one-step baselines.
Appendix~\ref{app:limitations} discusses the limitations of IDP. Improving demonstration coverage and extending the policy to multi-modal action generation are directions for future work.

\section*{Acknowledgments}
This work was supported by the National Natural Science Foundation of China (Project Number 62595774), MoE Key Laboratory of Intelligent Perception and Human-Machine Collaboration (KLIP-HuMaCo), HPC Platform of ShanghaiTech University.

\bibliographystyle{plainnat}
\bibliography{references}

\newpage

\begin{center}\Large\textbf{Appendix}\end{center}
\appendix
\section{Proofs and Derivations}\label{app:proofs}

\subsection{Proof of Proposition~\ref{prop:degeneration_explicit_drifting} (Degeneration of Explicit Drifting in Behavior Cloning)}\label{app:proof_prop1}

\begin{proof}[Proof]
We compute the empirical conditional drifting field explicitly under the Dirac assumptions and substitute it into the field-displaced regression objective.

We first restate the conditional drifting field. Conditioning on $o_i$, the attraction and repulsion components are
\begin{align}
V_{\hat{p}}^{+}(a \mid o_i)
&= \frac{\mathbb{E}_{y \sim \hat{p}_{\text{data}}(\cdot \mid o_i)}\!\left[ k(a,y)(y-a) \right]}
        {\mathbb{E}_{y \sim \hat{p}_{\text{data}}(\cdot \mid o_i)}\!\left[ k(a,y) \right]},
&
V_{\hat{q}}^{-}(a \mid o_i)
&= \frac{\mathbb{E}_{y \sim \hat{q}_{\theta}(\cdot \mid o_i)}\!\left[ k(a,y)(y-a) \right]}
        {\mathbb{E}_{y \sim \hat{q}_{\theta}(\cdot \mid o_i)}\!\left[ k(a,y) \right]},
\end{align}
where the empirical conditional distributions are $\hat{p}_{\text{data}}(y \mid o_i) = \delta(y - a_i^*)$ and $\hat{q}_{\theta}(y \mid o_i) = \delta(y - a)$ with $a = f_\theta(o_i, a_0, 0)$.

We next evaluate the attraction $V_{\hat{p}}^{+}$. By the sifting property of the Dirac measure,
\begin{align}
\mathbb{E}_{y \sim \hat{p}_{\text{data}}(\cdot \mid o_i)}[k(a,y)(y-a)]
  &= \int k(a,y)\,(y-a)\,\delta(y - a_i^*)\,dy = k(a, a_i^*)\,(a_i^* - a), \\
\mathbb{E}_{y \sim \hat{p}_{\text{data}}(\cdot \mid o_i)}[k(a,y)]
  &= \int k(a,y)\,\delta(y - a_i^*)\,dy = k(a, a_i^*).
\end{align}
Since $k$ is positive ($k(a, a_i^*) > 0$ for all $a$, by assumption), the ratio is well-defined and the kernel cancels out exactly:
\begin{equation}
V_{\hat{p}}^{+}(a \mid o_i) \;=\; \frac{k(a, a_i^*)\,(a_i^* - a)}{k(a, a_i^*)} \;=\; a_i^* - a.
\end{equation}
The kernel thus contributes nothing to the conditional attraction direction once the support of $\hat{p}_{\text{data}}(\cdot \mid o_i)$ collapses onto a single point.

Turning to the repulsion $V_{\hat{q}}^{-}$, the term reduces to a self-evaluation since $\hat{q}_\theta(\cdot \mid o_i)$ is concentrated on the prediction $a$:
\begin{align}
\mathbb{E}_{y \sim \hat{q}_\theta(\cdot \mid o_i)}[k(a,y)(y-a)]
  &= \int k(a,y)\,(y-a)\,\delta(y - a)\,dy = k(a, a)\,(a - a) = 0, \\
\mathbb{E}_{y \sim \hat{q}_\theta(\cdot \mid o_i)}[k(a,y)]
  &= \int k(a,y)\,\delta(y - a)\,dy = k(a, a).
\end{align}
Positivity gives $k(a,a)>0$, so the self-evaluated repulsion vanishes:
\begin{equation}
V_{\hat{q}}^{-}(a \mid o_i) \;=\; \frac{0}{k(a, a)} \;=\; 0.
\end{equation}

Composing the two components,
\begin{equation}
V_{\hat{p},\hat{q}}(a \mid o_i)
  \;=\; V_{\hat{p}}^{+}(a \mid o_i) - V_{\hat{q}}^{-}(a \mid o_i)
  \;=\; (a_i^* - a) - 0
  \;=\; a_i^* - a.
\end{equation}
At the current prediction, the singleton empirical estimate is the expert--prediction difference, independently of the positive kernel. The empirical support is fixed when evaluating this expression.

It remains to substitute this drifting field into the field-displaced regression objective. With $a = f_\theta(o_i, a_0, 0)$,
\begin{align}
\mathcal{L}_{\mathrm{Drift}}(\theta)
&= \mathbb{E}_{a_0 \sim p_0}\!\left\| f_\theta(o_i, a_0, 0) - \mathrm{sg}\!\left[ f_\theta(o_i, a_0, 0) + V_{\hat{p},\hat{q}}(f_\theta(o_i, a_0, 0) \mid o_i) \right] \right\|_2^2 \\
&= \mathbb{E}_{a_0 \sim p_0}\!\left\| a - \mathrm{sg}\!\left[ a + (a_i^* - a) \right] \right\|_2^2 \\
&= \mathbb{E}_{a_0 \sim p_0}\!\left\| a - \mathrm{sg}[a_i^*] \right\|_2^2.
\end{align}
Since $a_i^*$ is a fixed dataset target and not a function of $\theta$, the stop-gradient operator acts as the identity on it, yielding
\begin{equation}
\mathcal{L}_{\mathrm{Drift}}(\theta) \;=\; \mathbb{E}_{a_0 \sim p_0}\!\left\| f_\theta(o_i, a_0, 0) - a_i^* \right\|_2^2.
\end{equation}

\end{proof}

\subsection{Derivation of the Geometry-Induced Potential}\label{app:proof_prop2}
\begin{proof}
For fixed symmetric $M_i\succeq0$, let $e=a-a_i^*$. Differentiating the quadratic potential gives
\begin{equation}
\nabla_a E_i^{\mathrm{geo}}(a)=\tfrac12(M_i+M_i^\top)e=M_i e.
\label{eq:appendix_potential_gradient}
\end{equation}
Its negative gradient is therefore $M_i(a_i^*-a)$, the geometric correction in Eq.~\eqref{eq:geo_correction}. Adding the isotropic term gives $-\nabla_aE_i(a)=(I+M_i)(a_i^*-a)$. Since $I+M_i\succ0$, the expert action is the unique minimum of the complete potential as a function of the action variable with geometry fixed.
\end{proof}

\subsection{Interpretation of Conditional Expert Geometry}\label{app:kernel_smoothing}
For fixed nonnegative weights with $\sum_jw_{ij}=1$, let $\mu_i^w=\sum_jw_{ij}a_j^*$. Expanding Eq.~\eqref{eq:cond_expert_geometry} about this weighted mean cancels the cross terms:
\begin{equation}
G_i=\sum_jw_{ij}(a_j^*-\mu_i^w)(a_j^*-\mu_i^w)^\top
 +(\mu_i^w-a_i^*)(\mu_i^w-a_i^*)^\top.
\label{eq:target_centered_decomposition}
\end{equation}
The first term measures how related expert actions vary around their weighted mean. The second measures the offset of that mean from the current expert target $a_i^*$. Together, they describe action variation around the target used for supervision.

\subsection{General Multivariate Excess and Diagonal Instantiation}\label{app:diag_form}

In Sec.~4.3 of the main text we instantiated the geometry excess via a coordinate-wise comparison of normalized inverse standard deviations and noted, in passing, the existence of a more general multivariate form. Here we describe this conceptual extension and distinguish it from the implemented diagonal weights.

\paragraph{Full multivariate excess.}
Let $P_i^{\mathrm{cond}} := (G_i + \varepsilon I)^{-1}$ and $P^{\mathrm{ref}} := (\Sigma_{\mathrm{ref}} + \varepsilon I)^{-1}$ be the regularized conditional and reference precision matrices. To compare them on a common scale, we whiten $P_i^{\mathrm{cond}}$ against $P^{\mathrm{ref}}$:
\begin{equation}
R_i \;:=\; (P^{\mathrm{ref}})^{-1/2}\, P_i^{\mathrm{cond}}\, (P^{\mathrm{ref}})^{-1/2}.
\end{equation}
Eigenvalues $\lambda_d(R_i)$ of $R_i$ measure how much tighter the conditional precision is along each generalized direction relative to the reference. The natural matrix-valued positive excess is then
\begin{equation}
M_i^{\mathrm{full}} \;=\; U_i\, \mathrm{Diag}\!\Big[\,\big( \lambda_d(R_i) - 1 \big)_+\,\Big]\, U_i^\top,
\end{equation}
where $U_i$ is the orthogonal matrix of eigenvectors of $R_i$ and $(\cdot)_+ = \max(\cdot, 0)$.

\paragraph{Diagonal construction.}
Computing $M_i^{\mathrm{full}}$ requires a generalized eigenvalue decomposition (GEVD) of $P_i^{\mathrm{cond}}$ and $P^{\mathrm{ref}}$ at a cost of $\mathcal{O}(d_a^3)$ \emph{per training sample}, which is prohibitive in the inner loop of policy training. Under the diagonal restriction $G_i \approx \mathrm{Diag}(v_{i}^{\mathrm{cond}})$ and $\Sigma_{\mathrm{ref}} \approx \mathrm{Diag}(v^{\mathrm{ref}})$, both precision matrices become diagonal, $U_i = I$, and the eigenvalue ratio reduces to the per-coordinate inverse-variance ratio:
\begin{equation}
\lambda_d(R_i) \;=\; \frac{1\big/(v_{i,d}^{\mathrm{cond}} + \varepsilon)}{1\big/(v_d^{\mathrm{ref}} + \varepsilon)} \;=\; \frac{v_d^{\mathrm{ref}} + \varepsilon}{v_{i,d}^{\mathrm{cond}} + \varepsilon}.
\end{equation}
For the implemented diagonal weights, we normalize the conditional and reference inverse standard deviations across coordinates and retain their positive relative excess, as defined in Eqs.~\eqref{eq:normalized_scales}--\eqref{eq:geometry_excess}.

\paragraph{Why the diagonal approximation is justified here:}
\begin{itemize}
\item \emph{Computational cost.} The cost drops from $\mathcal{O}(d_a^3)$ to $\mathcal{O}(d_a)$ per sample, eliminating GEVD from the training inner loop entirely.
\item \emph{Use case for $M_i$.} The metric $M_i$ is consumed only via the contraction $r^\top M_i r$ for $r = a - a_i^*$. Coordinate-wise penalties along physically meaningful action axes (joint angles, end-effector velocities, gripper command) already constitute a strong inductive bias.
\end{itemize}

Each action-chunk entry has its own weight $m_{i,\ell,d}$.

\section{Algorithms}\label{app:algorithms}

Algorithm~\ref{alg:idp_training} summarizes IDP training. Geometry is computed from the current minibatch and held fixed during backpropagation. The policy parameters, including the observation encoder, are optimized using Eq.~\eqref{eq:idp_total_loss}.

\begin{algorithm}[htbp]
\caption{IDP Training}\label{alg:idp_training}
\begin{algorithmic}[1]
\Require Demonstrations $\mathcal D$, policy $f_\theta$, $t_*$, $\lambda_{\mathrm{prox}}$
\While{training}
    \State Sample a minibatch $\mathcal B$
    \State Compute $h_i$ using Eq.~\eqref{eq:obs_features}
    \State Compute row-standardized similarities and weights $w_{ij}$
    \State Compute local and reference action statistics
    \State Compute $M_i$ using Eqs.~\eqref{eq:normalized_scales}--\eqref{eq:geometry_excess}
    \State Sample $\epsilon_i\sim\mathcal N(0,I)$ and set $\tilde a_i=a_i^*+(1-t_*)\epsilon_i$
    \State $y_i\gets f_\theta(o_i,0,0)$ and $z_i\gets f_\theta(o_i,\tilde a_i,t_*)$
    \State Evaluate $\mathcal L$ using Eq.~\eqref{eq:idp_total_loss}
    \State Update $\theta$
\EndWhile
\State \textbf{Deployment:} $f_\theta(o,0,0)$
\end{algorithmic}
\end{algorithm}

All batch samples participate in softmax, including the current sample $i$. We evaluate two ways of removing it with fixed policy parameters on ToolHang-State and Box-Close. Post-softmax removal renormalizes the remaining weights. Pre-standardization removal excludes the sample before computing row statistics and softmax. Table~\ref{tab:self_audit} reports the relative $L_2$ changes in geometry.

\begin{table}[htbp]
\centering
\caption{Self-inclusion diagnostics. The last two columns report relative $L_2$ changes, summarized using the statistic shown for each task.}
\label{tab:self_audit}
\setlength{\tabcolsep}{4pt}
\begin{tabular}{llrrr}
\toprule
Task & Statistic & \makecell{Mean self\\weight (\%)} & \makecell{Post-softmax\\removal} & \makecell{Pre-standardization\\removal} \\
\midrule
ToolHang-State & Median & 0.52 & $9.2\times10^{-6}$ & $1.75\times10^{-3}$ \\
Box-Close & Mean & 2.09 & $2.7\times10^{-6}$ & $5.0\times10^{-3}$ \\
\bottomrule
\end{tabular}
\end{table}

\section{Detailed Experimental Setup}\label{app:setup}

\subsection{2D Benchmarks}\label{app:setup_2d}

We follow the protocol of MIP~\cite{pan2026adonoisingdispellingmyths} for all 2D experiments. The benchmark suite consists of six continuous-control manipulation tasks:
\begin{itemize}
\item \textbf{Robomimic Lift / Can / Square / Transport}: four manipulation primitives ranging from picking-and-lifting (Lift, easy) to two-arm transport (hardest). Each task is evaluated on \texttt{ph} (proficient-human, $200$ expert demonstrations) and \texttt{mh} (multi-human, $300$ demonstrations from operators of varying skill).
\item \textbf{Tool-Hang}: a long-horizon precision task where a hook tool must be inserted onto a wall mount. Notable for its narrow valid grasping pose region, which makes it a strong testbed for IDP's local geometric correction.
\item \textbf{PushT}: a 2D pushing task with an underactuated dynamic shape. The action manifold is broad, making the task relatively insensitive to anisotropic geometric constraints. Its action scatter is visualized in Fig.~\ref{fig:cov_pusht}.
\end{itemize}

For the task subsets reported in Tables~\ref{tab:2d_summary_state} and~\ref{tab:2d_summary_image}, we evaluate under the corresponding observation modalities: (i) \emph{state-based}, where the observation is the low-dimensional state of the simulator, and (ii) \emph{image-based}, where the observation consists of RGB rendering(s). We use the same dataset splits, action chunking horizons, and replan strides as MIP. Evaluation is run for $50$ environment initial conditions per seed, repeated over $3$ random seeds, and averaged across two architectures (Transformer and CNN) for both observation modalities.

\subsection{3D Benchmarks}\label{app:setup_3d}

For 3D pointcloud-conditioned manipulation we follow OFP~\cite{li2026onestepflowpolicyselfdistillation} and FlowPolicy~\cite{zhang2024flowpolicyenablingfastrobust}. The suite contains $56$ tasks organized as:
\begin{itemize}
\item \textbf{Adroit}~\cite{rajeswaran2017learning} (3 tasks): Hammer, Door, Pen. A $26$-DoF dexterous hand-arm, requiring high-precision contact-rich manipulation.
\item \textbf{DexArt}~\cite{bao2023dexart} (4 tasks): Laptop, Faucet, Bucket, Toilet. Articulated-object manipulation with diverse joint configurations.
\item \textbf{MetaWorld}~\cite{yu2020meta} (49 tasks): a broad set of parallel-gripper tasks. Following OFP, we partition the $49$ tasks into difficulty levels (28 Easy / 11 Medium / 5 Hard / 5 Very-Hard) for analysis.
\end{itemize}

Observations are point clouds of the workspace ($1024$ points per frame, sampled following DP3~\cite{ze20243d}). All training and evaluation hyperparameters that are not specific to IDP follow OFP / FlowPolicy unchanged. Each method is averaged over $3$ random seeds, and per-task standard deviations are reported in Table~\ref{tab:full_3d}.

\subsection{Real-world ``Pick Peach'' Setup}\label{app:setup_real}

\paragraph{Hardware.} We use a single-arm ALOHA platform equipped with a parallel-jaw gripper. Two RGB cameras provide visual feedback: one front-view and one head-mounted (egocentric).

\paragraph{Task.} The robot must pick up a peach placed at one of two designated positions and transport it onto a target plate. The task simultaneously tests spatial generalization (peach position varies within each side) and multi-modal modeling (left vs. right placement).

\paragraph{Demonstrations.} We collect teleoperated demonstrations under a strictly bimodal placement distribution: $50\%$ of trajectories place the peach on the left side of the plate and $50\%$ on the right. The two modes are exactly balanced, providing a clean test of whether single-step policies can preserve multi-modality without iterative stochastic injection.

\paragraph{Evaluation.} Each method is evaluated for $10$ trials, $5$ from each mode (left / right). A trial is counted as successful only if the peach is grasped, lifted, transported, and released onto the plate. We test under two observation modalities: state-based (\texttt{qpos} only, no visual feedback) and vision-based (both RGB streams). Results are summarized in Table~\ref{tab:real_robot} of the main text.

\subsection{Baseline Reproduction Details}\label{app:setup_baselines}

For 2D state- and image-based experiments, the baselines DP, MIP, MP1, and Flow follow MIP's protocol and code base unchanged. For 3D experiments, the multi-step DP3 / Flow baselines and the few-step CP / MP1 / FlowPolicy baselines are reproduced from the OFP repository on a single NVIDIA A40 to ensure fair compute comparison. Results marked with $^*$ in Table~\ref{tab:3d_summary} (and Tab.~\ref{tab:full_3d}) indicate this re-running protocol. Numbers without $^*$ are cited from the corresponding original papers (OFP / FlowPolicy / DP3) when reproduction was infeasible due to compute budget.

Naive Drifting estimates the empirical field from expert and generated samples and regresses to the resulting frozen target. IDP w/o geometry retains both training inputs and sets $M_i=0$. The deployment-call counts are reported in the result tables.
\section{Implementation Details and Hyperparameters}\label{app:impl}

\subsection{Architecture}\label{app:arch}

\paragraph{Policy backbone.}
For 2D experiments we adopt the Transformer and CNN backbones from MIP~\cite{pan2026adonoisingdispellingmyths} unchanged. Each task is run with both architectures and the success rates are then averaged. For 3D experiments we use the DP3 PointNet-based encoder~\cite{ze20243d} for the point-cloud observation and a Transformer policy network, identical to the setup used by OFP~\cite{li2026onestepflowpolicyselfdistillation} and FlowPolicy~\cite{zhang2024flowpolicyenablingfastrobust}.

\paragraph{Observation embedding $\phi_\theta$.}
The embedding $\phi_\theta(o_i)$ used to compute the neighborhood weights $w_{ij}$ is read from the final hidden representation of the observation encoder, prior to any task-specific head. Stop-gradient (\texttt{detach}) and $L_2$-normalization are applied to $\phi_\theta(o_i)$ before computing pairwise inner products, so that the geometry statistics are fixed during the current backward pass (Appendix~\ref{app:algorithms}).

\subsection{Hyperparameters}\label{app:hyper}

Table~\ref{tab:hyper} summarizes the training and geometry settings. Other quantities (learning rate, batch size, training steps, action-chunk horizon, and evaluation protocol) retain the original benchmark configurations.

\begin{table}[h]
\centering
\caption{Training and geometry settings. Values are fixed across the evaluated settings.}
\label{tab:hyper}
\setlength{\tabcolsep}{8pt}
\renewcommand{\arraystretch}{1.15}
\resizebox{\textwidth}{!}{
\begin{tabular}{l c l}
\toprule
\textbf{Symbol} & \textbf{Value} & \textbf{Description} \\
\midrule
$t_*$ & $0.9$ & Schedule scalar for the expert-proximal input. \\
$\lambda_{\mathrm{prox}}$ & $\big(t_*/(1-t_*)\big)^{2} = 81$ & Proximal loss weight (Eq.~\eqref{eq:idp_total_loss}). \\
Probe noise $\varepsilon_i$ & $\mathcal{N}(0, I)$ & Scaled by $(1{-}t_*)$ in $\tilde a_i = a_i^* + (1-t_*)\varepsilon_i$. \\
Neighborhood scope & full minibatch & $w_{ij}$ over all $j \in \{1, \ldots, B\}$ including $j = i$. \\
Softmax temperature & $1$ & Applied directly on row z-scored similarities (no extra temperature). \\
Numerical $\varepsilon$ & $10^{-6}$ & Used in $L_2$-normalization, variance flooring, and ratio denominators. \\
$v_d^{\mathrm{ref}}$ estimation & per-minibatch & $v_d^{\mathrm{ref}} = \tfrac{1}{B}\sum_i (a_{i,d}^* - \bar a_d^*)^2$, no EMA / no precompute. \\
Stop-gradient on $w_{ij}, M_i$ & yes & Geometry path does not flow gradient into $\theta$. \\
\bottomrule
\end{tabular}}
\end{table}

\paragraph{Implementation note on $\lambda_{\mathrm{prox}}$.}
We fix $t_*=0.9$ and $\lambda_{\mathrm{prox}}=[t_*/(1-t_*)]^2=81$ across all 2D, 3D, and real-world experiments, without per-task tuning.
The training inputs, relative time weighting, and 2D experimental configuration follow~\cite{pan2026adonoisingdispellingmyths}.

\subsection{Compute Resources}\label{app:compute}

The main benchmark experiments used an NVIDIA A40 GPU. Geometry computation consists of batch feature similarities, softmax weights, and coordinate-wise action statistics. The corresponding costs are $O(B^2d_h)$ and $O(B^2Hd_a)$.
We measure training step time, peak memory, and inference latency on the same machine with the same architecture and batch size. Endpoint regression uses the direct training input. IDP w/o geometry retains both training inputs and sets $M_i=0$. Inference latency is measured with batch size one.

\begin{table}[htbp]
\centering
\small
\caption{Training and inference costs. NFE is the number of policy evaluations at deployment.}
\label{tab:timing}
\setlength{\tabcolsep}{4pt}
\begin{tabular}{llrrrr}
\toprule
Input & Method & Step time (ms) & Peak memory (GiB) & NFE & Latency (ms) \\
\midrule
\multirow{4}{*}{State}
 & Endpoint regression & $78.5\pm4.6$ & 2.38 & 1 & 8.67 \\
 & IDP w/o geometry & $131.4\pm6.5$ & 4.39 & 1 & 8.79 \\
 & IDP & $130.7\pm4.1$ & 4.39 & 1 & 8.64 \\
 & MIP & $119.7\pm0.7$ & 4.39 & 2 & 13.73 \\
\midrule
\multirow{4}{*}{Image}
 & Endpoint regression & $254.0\pm0.3$ & 13.32 & 1 & 12.19 \\
 & IDP w/o geometry & $312.9\pm1.4$ & 15.33 & 1 & 12.18 \\
 & IDP & $313.1\pm1.3$ & 15.33 & 1 & 12.35 \\
 & MIP & $303.5\pm0.3$ & 15.33 & 2 & 17.21 \\
\bottomrule
\end{tabular}
\end{table}

IDP and IDP w/o geometry have similar training step times and the same peak memory, while IDP retains single-step deployment.

\section{Detailed Experimental Results}\label{app:results}

\subsection{Per-Task Results}\label{app:per_task}
\suppressfloats[t]

Tables~\ref{tab:full_2d} and~\ref{tab:full_2d_image} report the detailed 2D results for each architecture. Table~\ref{tab:full_3d} reports the 3D task results, and Table~\ref{tab:full_3d_expert} reports expert success rates.

\begin{table}[htbp]
    \centering
    \caption{Detailed 2D simulation results on \textbf{state-based} robotic manipulation tasks. For each task, we report the best success rate / the mean over the last five evaluations during training. Results are averaged over three training seeds.}
    \label{tab:full_2d}
    \setlength{\tabcolsep}{4.5pt}
    \resizebox{\textwidth}{!}{%
    \begin{tabular}{llccccccccccc}
    \toprule
    \multirow{2.5}{*}{Architecture} & \multirow{2.5}{*}{Method} & \multicolumn{2}{c}{Lift} & \multicolumn{2}{c}{Can} & \multicolumn{2}{c}{Square} & \multicolumn{2}{c}{Transport} & \multirow{2.5}{*}{Tool-Hang} & \multirow{2.5}{*}{Push-T} \\ \cmidrule{3-10}
                                  &                         & mh          & ph         & mh         & ph         & mh           & ph          & mh            & ph            &           &            \\ \midrule
\multirow{7}{*}{Transformer}  & DP                      & \textbf{1.00}/0.99   & \textbf{1.00}/\textbf{1.00}  & \textbf{0.97}/0.89  & \textbf{1.00}/0.97  & 0.74/0.67    & 0.95/0.90   & 0.37/0.25     & 0.64/0.54     & 0.77/\textbf{0.68} & 0.92/0.88   \\
                              & Flow                    & \textbf{1.00}/\textbf{1.00}   & \textbf{1.00}/\textbf{1.00}  & 0.92/0.87  & \textbf{1.00}/\textbf{0.98}  & 0.69/0.53    & 0.94/0.88   & 0.23/0.14     & \textbf{0.71}/0.56     & 0.58/0.46 & 0.94/\textbf{0.92}  \\
                              & MIP                     & \textbf{1.00}/0.99   & \textbf{1.00}/\textbf{1.00}  & 0.94/\textbf{0.91}  & \textbf{1.00}/\textbf{0.98}  & \textbf{0.77}/\textbf{0.69}    & \textbf{0.98}/\textbf{0.91}   & \textbf{0.40}/\textbf{0.26}     & 0.69/0.54     & 0.73/0.61 & 0.94/0.91   \\
                              & CP                      & 0.98/0.96   & \textbf{1.00}/\textbf{1.00}  & 0.78/0.73  & 0.99/0.93  & 0.43/0.40    & 0.92/0.85   & 0.02/0.01     & 0.34/0.27     & 0.38/0.33 & 0.94/0.91   \\
                              & MP1                     & \textbf{1.00}/0.99   & \textbf{1.00}/\textbf{1.00}  & 0.94/0.80  & \textbf{1.00}/0.97  & 0.64/0.46    & 0.94/0.84   & 0.14/0.05     & 0.47/0.33     & 0.70/0.56 & 0.93/0.88   \\
                              & Naive Drifting          & \textbf{1.00}/0.99   & \textbf{1.00}/\textbf{1.00}  & 0.94/0.86  & 0.99/0.96  & 0.56/0.47    & \textbf{0.98}/0.89   & 0.15/0.11     & 0.61/0.55     & 0.46/0.37 & 0.93/0.88   \\
                              & IDP                     & \textbf{1.00}/0.99   & \textbf{1.00}/\textbf{1.00}  & \textbf{0.97}/\textbf{0.91}  & \textbf{1.00}/0.97  & 0.76/0.63    & 0.96/0.90   & 0.34/0.23     & 0.65/\textbf{0.57}     & \textbf{0.78}/\textbf{0.68} & \textbf{0.97}/\textbf{0.92}   \\ \midrule
\multirow{7}{*}{CNN}          & DP                      & \textbf{1.00}/\textbf{1.00}   & \textbf{1.00}/\textbf{1.00}  & \textbf{1.00}/0.95  & \textbf{1.00}/\textbf{0.99}  & 0.88/0.77    & 0.97/\textbf{0.93}   & \textbf{0.53}/0.43     & \textbf{0.81}/\textbf{0.71}     & \textbf{0.91}/\textbf{0.74} & \textbf{0.98}/0.96   \\
                              & Flow                    & \textbf{1.00}/0.99   & \textbf{1.00}/\textbf{1.00}  & \textbf{1.00}/\textbf{0.96}  & \textbf{1.00}/\textbf{0.99}  & 0.83/0.77    & 0.96/0.91   & 0.43/0.31     & 0.78/0.63     & 0.81/0.69 & \textbf{0.98}/0.96   \\
                              & MIP                     & \textbf{1.00}/\textbf{1.00}   & \textbf{1.00}/\textbf{1.00}  & 0.99/\textbf{0.96}  & \textbf{1.00}/0.98  & 0.83/0.69    & 0.94/0.84   & 0.48/\textbf{0.45}     & 0.73/0.62     & 0.63/0.49 & \textbf{0.98}/\textbf{0.97}   \\
                              & CP                      & \textbf{1.00}/0.98   & \textbf{1.00}/\textbf{1.00}  & 0.88/0.76  & \textbf{1.00}/0.98  & 0.56/0.49    & 0.84/0.80   & 0.01/0.00     & 0.29/0.22     & 0.49/0.33 & 0.94/0.91   \\
                              & MP1                     & \textbf{1.00}/\textbf{1.00}   & \textbf{1.00}/\textbf{1.00}  & 0.99/0.94  & \textbf{1.00}/0.98    & \textbf{0.89}/\textbf{0.78}   & \textbf{0.98}/0.91     & 0.33/0.22     & 0.57/0.44 & 0.81/0.70  &  0.97/0.96 \\
                              & Naive Drifting          & \textbf{1.00}/0.99   & \textbf{1.00}/\textbf{1.00}  & 0.99/0.95  & \textbf{1.00}/0.97  & 0.83/0.74    & 0.96/0.90   & 0.46/0.33     & 0.74/0.61     & 0.59/0.48 & 0.96/0.94   \\
                              & IDP                     & \textbf{1.00}/\textbf{1.00}   & \textbf{1.00}/\textbf{1.00}  & \textbf{1.00}/\textbf{0.96}  & \textbf{1.00}/0.98  & 0.88/0.72    & 0.94/0.86   & 0.48/0.36     & 0.78/0.60     & 0.84/0.53 & \textbf{0.98}/0.91    \\ \bottomrule                
    \end{tabular}}
\end{table}

\begin{table}[ht]
    \centering
    \caption{Detailed 2D simulation results on \textbf{image-based} robotic manipulation tasks. For each task, we report the best success rate / the mean over the last five evaluations during training. Results are averaged over three training seeds.}
    \label{tab:full_2d_image}
    \setlength{\tabcolsep}{4.5pt}
    \resizebox{0.95\textwidth}{!}{%
    \begin{tabular}{llcccccccc}
    \toprule
    \multirow{2.5}{*}{Architecture} & \multirow{2.5}{*}{Method} & \multicolumn{2}{c}{Lift} & \multicolumn{2}{c}{Can} & \multicolumn{2}{c}{Square} & \multirow{2.5}{*}{Tool-Hang} & \multirow{2.5}{*}{Push-T} \\ \cmidrule{3-8}
                                  &                         & mh          & ph         & mh         & ph         & mh           & ph          &            &            \\ \midrule
\multirow{7}{*}{Transformer}   & DP                      & \textbf{1.00}/0.98   & \textbf{1.00}/0.99  & \textbf{0.97}/\textbf{0.92}  & \textbf{0.99}/\textbf{0.96}  & \textbf{0.85}/0.74    & 0.97/0.89       & 0.35/0.17 & \textbf{0.93}/0.87  \\
                              & Flow                    & \textbf{1.00}/0.99   & \textbf{1.00}/0.99  & 0.92/0.85  & \textbf{0.99}/0.94  & 0.81/0.73    & 0.97/0.88   & 0.35/\textbf{0.26} & 0.90/0.84  \\
                              & MIP                     & \textbf{1.00}/0.97   & \textbf{1.00}/\textbf{1.00}  & 0.91/0.86  & \textbf{0.99}/0.94  & \textbf{0.85}/0.73    & \textbf{0.98}/0.86   & 0.33/0.14 & \textbf{0.93}/\textbf{0.89}  \\
                              & CP                      & \textbf{1.00}/0.98   & \textbf{1.00}/\textbf{1.00}  & 0.78/0.64  & 0.86/0.61  & 0.72/0.63    & 0.92/0.87     & 0.25/0.16 & 0.86/0.82   \\
                              & MP1                     & \textbf{1.00}/\textbf{1.00}   & \textbf{1.00}/0.99  & 0.91/0.82  & 0.79/0.69  & \textbf{0.85}/\textbf{0.75}    & 0.95/\textbf{0.91}   & 0.30/0.13 & 0.83/0.72  \\
                              & Naive Drifting          & \textbf{1.00}/0.99   & \textbf{1.00}/0.99  & 0.95/0.90  & 0.97/0.86  & \textbf{0.85}/0.73    & 0.94/0.89   & \textbf{0.37}/0.19 & 0.90/0.86  \\
                              & Ours                    & \textbf{1.00}/0.98   & \textbf{1.00}/\textbf{1.00}  & \textbf{0.97}/0.91  & \textbf{0.99}/0.94  & 0.81/0.74    & 0.96/0.89   & \textbf{0.37}/0.21 & 0.92/0.87  \\ \midrule

\multirow{7}{*}{CNN}          & DP                      & \textbf{1.00}/\textbf{0.99}   & \textbf{1.00}/0.99  & 0.82/0.73  & 0.98/0.92  & 0.89/0.78    & 0.97/\textbf{0.93}     & \textbf{0.61}/\textbf{0.48} & 0.92/\textbf{0.89}  \\
                              & Flow                    & \textbf{1.00}/0.98   & \textbf{1.00}/0.99  & \textbf{0.93}/0.81  & \textbf{0.99}/0.93  & \textbf{0.91}/\textbf{0.83}    & \textbf{0.99}/0.92   & 0.59/0.45 & \textbf{0.93}/0.88  \\
                              & MIP                     & \textbf{1.00}/0.96   & \textbf{1.00}/0.99  & 0.87/0.81  & 0.96/0.92  & 0.79/0.68    & 0.91/0.81   & 0.60/0.42 & 0.81/0.79  \\
                              & CP                      & \textbf{1.00}/0.98   & \textbf{1.00}/\textbf{1.00}  & 0.76/0.62  & 0.75/0.59  & 0.67/0.56    & 0.92/0.86   & 0.19/0.09 & 0.89/0.85 \\
                              & MP1                     & \textbf{1.00}/0.95   & \textbf{1.00}/0.99  & 0.83/0.72  & 0.74/0.66  & 0.86/0.80    & 0.97/0.88   & 0.55/0.40 & 0.90/0.86  \\
                              & Naive Drifting          & 0.99/0.94   & \textbf{1.00}/0.98  & 0.85/0.69  & 0.93/0.83  & 0.75/0.64    & 0.90/0.80   & 0.55/0.33 & 0.91/0.87  \\
                              & Ours                    & \textbf{1.00}/0.96   & \textbf{1.00}/0.99  & 0.92/\textbf{0.85}  & \textbf{0.99}/\textbf{0.94}  & 0.83/0.75    & 0.92/0.84   & 0.59/0.45 & 0.92/0.86  \\
    \bottomrule                       
    \end{tabular}}
\end{table}

\clearpage

\begin{table}[!htbp]
\centering
\caption{Detailed 3D simulation results on \textbf{Adroit, DexArt, and Meta-World.} Results are averaged over 3 random seeds and reported as mean $\pm$ standard deviation. The best result in each column is highlighted in \textbf{bold}.}
\label{tab:full_3d}

\begingroup
\scriptsize
\setlength{\tabcolsep}{1.6pt}
\renewcommand{\arraystretch}{0.93}
\renewcommand{\tabularxcolumn}[1]{m{#1}}

\begin{tabularx}{\textwidth}{@{}lc*{3}{>{\centering\arraybackslash}X}*{4}{>{\centering\arraybackslash}X}*{4}{>{\centering\arraybackslash}X}@{}}
\toprule
\multirow{3.5}{*}{\textbf{Methods}} & \multirow{3.5}{*}{\textbf{NFE}} & \multicolumn{3}{c}{\textbf{Adroit}} & \multicolumn{4}{c}{\textbf{DexArt}} & \multicolumn{4}{c}{\textbf{Meta-World (Easy)}} \\
\cmidrule(lr){3-5} \cmidrule(lr){6-9} \cmidrule(lr){10-13}
& 
& Hammer
& Door
& Pen
& Laptop
& Faucet
& Bucket
& Toilet
& \shortstack[c]{Coffee\\Button}
& \shortstack[c]{Dial\\Turn}
& \shortstack[c]{Door\\Close}
& \shortstack[c]{Door\\Lock} \\
\midrule
DP3 & 10 & \ddbf{100}{0} & \dd{58}{5} & \dd{53}{3} & \dd{83}{2} & \dd{40}{2} & \dd{25}{1} & \dd{74}{2} & \ddbf{100}{0} & \dd{91}{3} & \ddbf{100}{0} & \ddbf{100}{0} \\
Flow & 10 &\dd{99}{1} & \dd{69}{3} & \dd{45}{3} & \dd{85}{0} & \dd{42}{4} & \dd{26}{3} & \dd{74}{3} & \ddbf{100}{0} & \dd{81}{5} & \ddbf{100}{0} & \ddbf{100}{0} \\
CP & 1 & \dd{97}{2} & \dd{66}{0} & \dd{54}{2} & \dd{76}{11} & \dd{38}{0} & \dd{25}{5} & \dd{73}{4} & \dd{90}{14} & \dd{82}{2} & \dd{93}{9} & \dd{74}{34} \\
MP1 & 1 & \dd{99}{1} & \dd{48}{2} & \ddbf{58}{4} & \dd{82}{3} & \dd{46}{2} & \dd{26}{3} & \dd{69}{3} & \ddbf{100}{0} & \dd{91}{5} & \ddbf{100}{0} & \ddbf{100}{0} \\
FlowPolicy & 1 & \dd{98}{2} & \dd{66}{1} & \dd{55}{2} & \dd{86}{2} & \dd{40}{1} & \dd{28}{0} & \dd{76}{3} & \ddbf{100}{0} & \dd{81}{5} & \ddbf{100}{0} & \dd{99}{1} \\
Naive Drifting & 1 & \ddbf{100}{0} & \ddbf{67}{3} & \dd{43}{5} & \dd{95}{4} & \dd{46}{3} & \dd{29}{1} & \dd{82}{6} & \ddbf{100}{0} & \ddbf{97}{2} & \ddbf{100}{0} & \ddbf{100}{0} \\
Ours & 1 & \ddbf{100}{0} & \ddbf{67}{2} & \dd{52}{4} & \ddbf{96}{3} & \ddbf{48}{1} & \ddbf{30}{0} & \ddbf{83}{6} & \ddbf{100}{0} & \ddbf{97}{4} & \ddbf{100}{0} & \ddbf{100}{0} \\
\bottomrule
\end{tabularx}

\vspace{0.12em}

\begin{tabularx}{\textwidth}{@{}lc*{9}{>{\centering\arraybackslash}X}@{}}
\toprule
\multirow{3.5}{*}{\textbf{Methods}} & \multirow{3.5}{*}{\textbf{NFE}} & \multicolumn{9}{c}{\textbf{Meta-World (Easy)}} \\
\cmidrule(lr){3-11}
& 
& \shortstack[c]{Button\\Press}
& \shortstack[c]{ButtonPress\\Topdown}
& \shortstack[c]{ButtonPress\\Topdown Wall}
& \shortstack[c]{ButtonPress\\Wall}
& \shortstack[c]{Door\\Open}
& \shortstack[c]{Door\\Unlock}
& \shortstack[c]{Drawer\\Close}
& \shortstack[c]{Drawer\\Open}
& \shortstack[c]{Faucet\\Close} \\
\midrule
DP3 & 10 & \ddbf{100}{0} & \ddbf{100}{0} & \ddbf{100}{0} & \ddbf{100}{0} & \ddbf{100}{0} & \ddbf{100}{0} & \ddbf{100}{0} & \ddbf{100}{0} & \ddbf{100}{0}  \\
Flow & 10 & \ddbf{100}{0} & \ddbf{100}{0} & \ddbf{100}{0} & \ddbf{100}{0} & \ddbf{100}{0} & \ddbf{100}{0} & \ddbf{100}{0} & \ddbf{100}{0} & \ddbf{100}{0}  \\
CP & 1 & \dd{95}{8} & \dd{70}{42} & \dd{71}{41} & \ddbf{100}{0} & \dd{85}{22} & \ddbf{100}{0} & \ddbf{100}{0} & \dd{92}{12} & \ddbf{100}{0} \\
MP1 & 1 & \ddbf{100}{0} & \ddbf{100}{0} & \ddbf{100}{0} & \ddbf{100}{0} & \ddbf{100}{0} & \ddbf{100}{0} & \ddbf{100}{0} & \ddbf{100}{0} & \ddbf{100}{0} \\
FlowPolicy & 1 & \ddbf{100}{0} & \ddbf{100}{0} & \ddbf{100}{0} & \ddbf{100}{0} & \ddbf{100}{0} & \ddbf{100}{0} & \ddbf{100}{0} & \ddbf{100}{0} & \ddbf{100}{0}  \\
Naive Drifting & 1 & \ddbf{100}{0} & \ddbf{100}{0} & \ddbf{100}{0} & \ddbf{100}{0} & \ddbf{100}{0} & \ddbf{100}{0} & \ddbf{100}{0} & \ddbf{100}{0} & \ddbf{100}{0}  \\
Ours & 1 & \ddbf{100}{0} & \ddbf{100}{0} & \ddbf{100}{0} & \ddbf{100}{0} & \ddbf{100}{0} & \ddbf{100}{0} & \ddbf{100}{0} & \ddbf{100}{0} & \ddbf{100}{0}  \\
\bottomrule
\end{tabularx}

\vspace{0.12em}

\begin{tabularx}{\textwidth}{@{}lc*{10}{>{\centering\arraybackslash}X}@{}}
\toprule
\multirow{3.5}{*}{\textbf{Methods}} & \multirow{3.5}{*}{\textbf{NFE}} & \multicolumn{10}{c}{\textbf{Meta-World (Easy)}} \\
\cmidrule(lr){3-12}
& 
& \shortstack[c]{Faucet\\Open}
& \shortstack[c]{Handle\\Press}
& \shortstack[c]{Handle\\Pull}
& \shortstack[c]{Handle\\Press Side}
& \shortstack[c]{Handle\\Pull Side}
& \shortstack[c]{Lever\\Pull}
& \shortstack[c]{Plate\\Slide}
& \shortstack[c]{Plate Slide\\Back}
& \shortstack[c]{Plate Slide\\Back Side}
& \shortstack[c]{Plate Slide\\Side} \\
\midrule
DP3 & 10 & \ddbf{100}{0} & \dd{98}{1} & \dd{59}{11} & \ddbf{100}{0} & \dd{66}{5} & \dd{85}{2} & \ddbf{100}{0} & \ddbf{100}{0} & \ddbf{100}{0} & \ddbf{100}{0} \\
Flow & 10 & \ddbf{100}{0} & \ddbf{100}{0} & \dd{72}{8} & \ddbf{100}{0} & \ddbf{72}{2} & \dd{85}{1} & \ddbf{100}{0} & \ddbf{100}{0} & \ddbf{100}{0} & \ddbf{100}{0} \\
CP & 1 & \dd{85}{21} & \dd{91}{10} & \dd{66}{0} & \dd{88}{16} & \dd{17}{0} & \dd{65}{25} & \dd{74}{37} & \dd{93}{9} & \dd{92}{11} & \dd{99}{1} \\
MP1 & 1 & \ddbf{100}{0} & \dd{98}{3} & \dd{60}{9} & \ddbf{100}{0} & \dd{49}{6} & \dd{84}{4} & \ddbf{100}{0} & \ddbf{100}{0} & \ddbf{100}{0} & \ddbf{100}{0} \\
FlowPolicy & 1 & \ddbf{100}{0} & \dd{97}{3} & \dd{64}{1} & \ddbf{100}{0} & \dd{41}{0} & \dd{79}{4} & \ddbf{100}{0} & \ddbf{100}{0} & \ddbf{100}{0} & \ddbf{100}{0} \\
Naive Drifting & 1 & \ddbf{100}{0} & \dd{98}{2} & \dd{68}{16} & \ddbf{100}{0} & \dd{65}{7} & \dd{88}{9} & \ddbf{100}{0} & \ddbf{100}{0} & \ddbf{100}{0} & \ddbf{100}{0} \\
Ours & 1 & \ddbf{100}{0} & \ddbf{100}{0} & \ddbf{80}{1} & \ddbf{100}{0} & \dd{67}{4} & \ddbf{93}{2} & \ddbf{100}{0} & \ddbf{100}{0} & \ddbf{100}{0} & \ddbf{100}{0} \\
\bottomrule
\end{tabularx}

\vspace{0.12em}

\begin{tabularx}{\textwidth}{@{}lc*{5}{>{\centering\arraybackslash}X}*{4}{>{\centering\arraybackslash}X}@{}}
\toprule
\multirow{3.5}{*}{\textbf{Methods}} & \multirow{3.5}{*}{\textbf{NFE}} & \multicolumn{5}{c}{\textbf{Meta-World (Easy)}} & \multicolumn{4}{c}{\textbf{Meta-World (Medium)}} \\
\cmidrule(lr){3-7} \cmidrule(lr){8-11}
& 
& \shortstack[c]{Reach}
& \shortstack[c]{Reach\\Wall}
& \shortstack[c]{Window\\Close}
& \shortstack[c]{Window\\Open}
& \shortstack[c]{Peg Unplug\\Side}
& \shortstack[c]{Basketball}
& \shortstack[c]{Bin\\Picking}
& \shortstack[c]{Box\\Close}
& \shortstack[c]{Coffee\\Pull} \\
\midrule
DP3 & 10 & \dd{25}{2} & \dd{67}{1} & \ddbf{100}{0} & \ddbf{100}{0} & \dd{72}{1} & \dd{100}{0} & \dd{48}{18} & \dd{69}{5} & \dd{88}{5} \\
Flow & 10 & \dd{28}{1} & \dd{73}{2} & \ddbf{100}{0} & \ddbf{100}{0} & \dd{71}{1} & \ddbf{100}{0} & \dd{37}{13} & \dd{67}{1} & \dd{88}{2} \\
CP & 1 & \dd{27}{1} & \dd{56}{16} & \dd{96}{6} & \dd{75}{33} & \dd{57}{23} & \dd{96}{0} & \ddbf{48}{0} & \dd{60}{18} & \dd{72}{22} \\
MP1 & 1 & \dd{28}{1} & \ddbf{76}{2} & \ddbf{100}{0} & \ddbf{100}{0} & \dd{71}{3} & \ddbf{100}{0} & \dd{44}{15} & \dd{69}{10} & \ddbf{91}{2} \\
FlowPolicy & 1 & \dd{30}{1} & \dd{70}{2} & \ddbf{100}{0} & \ddbf{99}{0} & \dd{71}{6} & \dd{86}{8} & \dd{35}{11} & \ddbf{74}{2} & \dd{88}{1} \\
Naive Drifting & 1 & \ddbf{48}{2} & \dd{70}{4} & \ddbf{100}{0} & \ddbf{100}{0} & \dd{62}{9} & \dd{100}{0} & \dd{35}{21} & \dd{53}{10} & \dd{85}{4} \\
Ours & 1 & \dd{30}{3} & \dd{67}{6} & \ddbf{100}{0} & \ddbf{100}{0} & \ddbf{77}{3} & \ddbf{100}{0} & \dd{40}{6} & \dd{70}{4} & \dd{85}{1} \\
\bottomrule
\end{tabularx}

\vspace{0.12em}

\begin{tabularx}{\textwidth}{@{}lc*{7}{>{\centering\arraybackslash}X}*{2}{>{\centering\arraybackslash}X}@{}}
\toprule
\multirow{3.5}{*}{\textbf{Methods}} & \multirow{3.5}{*}{\textbf{NFE}} & \multicolumn{7}{c}{\textbf{Meta-World (Medium)}} & \multicolumn{2}{c}{\textbf{Meta-World (Hard)}} \\
\cmidrule(lr){3-9} \cmidrule(lr){10-11}
& 
& \shortstack[c]{Coffee\\Push}
& \shortstack[c]{Hammer}
& \shortstack[c]{Peg Insert\\Side}
& \shortstack[c]{Push\\Wall}
& \shortstack[c]{Soccer}
& \shortstack[c]{Sweep}
& \shortstack[c]{Sweep\\Into}
& \shortstack[c]{Hand\\Insert}
& \shortstack[c]{Pick Out\\of Hole} \\
\midrule
DP3 & 10 & \dd{89}{3} & \dd{87}{3} & \dd{83}{4} & \dd{88}{7} & \dd{26}{2} & \dd{95}{5} & \dd{46}{14} & \dd{15}{2} & \ddbf{51}{0} \\
Flow & 10 & \dd{89}{2} & \dd{88}{4} & \ddbf{87}{7} & \dd{90}{2} & \ddbf{31}{5} & \dd{96}{1} & \ddbf{48}{18} & \dd{17}{2} & \dd{46}{6} \\
CP & 1 & \dd{86}{14} & \dd{71}{25} & \dd{72}{0} & \dd{47}{0} & \dd{20}{8} & \dd{58}{39} & \dd{23}{8} & \dd{17}{4} & \dd{30}{0} \\
MP1 & 1 & \dd{92}{1} & \ddbf{90}{5} & \dd{86}{8} & \dd{89}{3} & \dd{30}{2} & \dd{95}{3} & \dd{43}{13} & \dd{16}{2} & \dd{38}{12} \\
FlowPolicy & 1 & \ddbf{93}{5} & \dd{85}{4} & \dd{75}{6} & \dd{51}{3} & \dd{29}{2} & \dd{84}{5} & \dd{37}{21} & \ddbf{20}{1} & \dd{22}{10} \\
Naive Drifting & 1 & \dd{92}{2} & \dd{84}{3} & \dd{85}{4} & \dd{93}{6} & \dd{33}{8} & \dd{93}{5} & \dd{25}{8} & \dd{12}{2} & \dd{40}{11} \\
Ours & 1 & \dd{92}{3} & \ddbf{90}{4} & \dd{79}{5} & \ddbf{95}{2} & \dd{27}{1} & \ddbf{100}{0} & \dd{32}{17} & \ddbf{20}{6} & \dd{44}{4} \\
\bottomrule
\end{tabularx}

\vspace{0.12em}

\begin{tabularx}{\textwidth}{@{}lc*{3}{>{\centering\arraybackslash}X}*{5}{>{\centering\arraybackslash}X}>{\centering\arraybackslash}X@{}}
\toprule
\multirow{3.5}{*}{\textbf{Methods}} & \multirow{3.5}{*}{\textbf{NFE}} & \multicolumn{3}{c}{\textbf{Meta-World (Hard)}} & \multicolumn{5}{c}{\textbf{Meta-World (Very Hard)}} & \multirow{3.5}{*}{\textbf{Avg.}} \\
\cmidrule(lr){3-5} \cmidrule(lr){6-10}
& 
& \shortstack[c]{Assembly}
& \shortstack[c]{Pick\\Place}
& \shortstack[c]{Push}
& \shortstack[c]{Shelf\\Place}
& \shortstack[c]{Disassemble}
& \shortstack[c]{Stick\\Pull}
& \shortstack[c]{Stick\\Push}
& \shortstack[c]{Pick Place\\Wall}
& \\
\midrule
DP3 & 10 & \dd{94}{1} & \dd{54}{3} & \dd{57}{3} & \dd{41}{15} & \dd{78}{3} & \dd{68}{1} & \ddbf{100}{0} & \dd{72}{4} & \dd{79.4}{2.4} \\
Flow & 10 & \dd{94}{1} & \dd{61}{2} & \dd{58}{4} & \dd{41}{9} & \dd{74}{1} & \dd{67}{3} & \ddbf{100}{0} & \dd{72}{2} & \dd{79.9}{2.2} \\
CP & 1 & \dd{74}{25} & \dd{51}{19} & \dd{35}{24} & \dd{36}{4} & \dd{60}{2} & \dd{48}{23} & \dd{99}{1} & \dd{67}{11} & \dd{68.4}{11.9} \\
MP1 & 1 & \dd{96}{1} & \ddbf{65}{1} & \ddbf{73}{7} & \dd{44}{12} & \dd{70}{3} & \ddbf{72}{4} & \ddbf{100}{0} & \ddbf{78}{2} & \dd{79.7}{2.8} \\
FlowPolicy & 1 & \dd{96}{3} & \dd{60}{4} & \dd{58}{6} & \dd{39}{6} & \dd{63}{4} & \dd{64}{1} & \ddbf{100}{0} & \dd{69}{4} & \dd{76.9}{2.5} \\
Naive Drifting & 1 & \ddbf{97}{2} & \ddbf{65}{8} & \dd{57}{6} & \ddbf{45}{4} & \dd{90}{4} & \dd{58}{12} & \ddbf{100}{0} & \dd{73}{6} & \dd{79.8}{3.5} \\
Ours & 1 & \dd{96}{3} & \dd{57}{4} & \dd{65}{4} & \dd{43}{2} & \ddbf{91}{1} & \dd{69}{2} & \ddbf{100}{0} & \dd{68}{2} & \ddbf{81.3}{2.0} \\
\bottomrule
\end{tabularx}

\endgroup
\end{table}


\begin{table}[htbp]
\centering
\caption{
\textbf{Success rates of experts on 3D simulation tasks.} We evaluate 200 episodes for each task. Results are averaged over 3 random seeds and reported as mean $\pm$ standard deviation.}
\label{tab:full_3d_expert}

\begingroup
\scriptsize
\setlength{\tabcolsep}{1.6pt}
\renewcommand{\arraystretch}{0.93}
\renewcommand{\tabularxcolumn}[1]{m{#1}}

\begin{tabularx}{\textwidth}{@{}l*{3}{>{\centering\arraybackslash}X}*{4}{>{\centering\arraybackslash}X}*{4}{>{\centering\arraybackslash}X}@{}}
\toprule
\multirow{3.5}{*}{\textbf{Methods}} & \multicolumn{3}{c}{\textbf{Adroit}}  & \multicolumn{4}{c}{\textbf{DexArt}} & \multicolumn{4}{c}{\textbf{Meta-World (Easy)}} \\
\cmidrule(lr){2-4} \cmidrule(lr){5-8} \cmidrule(lr){9-12}
& Hammer
& Door
& Pen
& Laptop
& Faucet
& Bucket
& Toilet
& \shortstack[c]{Coffee\\Button}
& \shortstack[c]{Dial\\Turn}
& \shortstack[c]{Door\\Close}
& \shortstack[c]{Door\\Lock} \\
\midrule
Expert & \dd{100}{0} & \dd{97}{1} & \dd{96}{1} & \dd{91}{1} & \dd{54}{4} & \dd{67}{2} & \dd{76}{1} & \dd{100}{0} & \dd{100}{0} & \dd{100}{0} & \dd{100}{0} \\
\bottomrule
\end{tabularx}

\vspace{0.12em}

\begin{tabularx}{\textwidth}{@{}l*{9}{>{\centering\arraybackslash}X}@{}}
\toprule
\multirow{3.5}{*}{\textbf{Methods}} & \multicolumn{9}{c}{\textbf{Meta-World (Easy)}} \\
\cmidrule(lr){2-10}
& \shortstack[c]{Button\\Press}
& \shortstack[c]{ButtonPress\\Topdown}
& \shortstack[c]{ButtonPress\\Topdown Wall}
& \shortstack[c]{ButtonPress\\Wall}
& \shortstack[c]{Door\\Open}
& \shortstack[c]{Door\\Unlock}
& \shortstack[c]{Drawer\\Close}
& \shortstack[c]{Drawer\\Open}
& \shortstack[c]{Faucet\\Close} \\
\midrule
Expert & \dd{100}{0} & \dd{100}{0} & \dd{100}{0} & \dd{97}{2} & \dd{97}{2} & \dd{100}{0} & \dd{100}{0} & \dd{100}{0} & \dd{100}{0} \\
\bottomrule
\end{tabularx}

\vspace{0.12em}

\begin{tabularx}{\textwidth}{@{}l*{10}{>{\centering\arraybackslash}X}@{}}
\toprule
\multirow{3.5}{*}{\textbf{Methods}} & \multicolumn{10}{c}{\textbf{Meta-World (Easy)}} \\
\cmidrule(lr){2-11}
& \shortstack[c]{Faucet\\Open}
& \shortstack[c]{Handle\\Press}
& \shortstack[c]{Handle\\Pull}
& \shortstack[c]{Handle\\Press Side}
& \shortstack[c]{Handle\\Pull Side}
& \shortstack[c]{Lever\\Pull}
& \shortstack[c]{Plate\\Slide}
& \shortstack[c]{Plate Slide\\Back}
& \shortstack[c]{Plate Slide\\Back Side}
& \shortstack[c]{Plate Slide\\Side} \\
\midrule
Expert & \dd{100}{0} & \dd{100}{0} & \dd{100}{0} & \dd{100}{0} & \dd{100}{0} & \dd{98}{0} & \dd{100}{0} & \dd{100}{0} & \dd{100}{0} & \dd{100}{0} \\
\bottomrule
\end{tabularx}

\vspace{0.12em}

\begin{tabularx}{\textwidth}{@{}l*{5}{>{\centering\arraybackslash}X}*{4}{>{\centering\arraybackslash}X}@{}}
\toprule
\multirow{3.5}{*}{\textbf{Methods}} & \multicolumn{5}{c}{\textbf{Meta-World (Easy)}} & \multicolumn{4}{c}{\textbf{Meta-World (Medium)}} \\
\cmidrule(lr){2-6} \cmidrule(lr){7-10}
& \shortstack[c]{Reach}
& \shortstack[c]{Reach\\Wall}
& \shortstack[c]{Window\\Close}
& \shortstack[c]{Window\\Open}
& \shortstack[c]{Peg Unplug\\Side}
& \shortstack[c]{Basketball}
& \shortstack[c]{Bin\\Picking}
& \shortstack[c]{Box\\Close}
& \shortstack[c]{Coffee\\Pull} \\
\midrule
Expert & \dd{100}{0} & \dd{100}{0} & \dd{100}{0} & \dd{100}{0} & \dd{99}{1} & \dd{100}{0} & \dd{96}{1} & \dd{88}{2} & \dd{100}{0} \\
\bottomrule
\end{tabularx}

\vspace{0.12em}

\begin{tabularx}{\textwidth}{@{}l*{7}{>{\centering\arraybackslash}X}*{2}{>{\centering\arraybackslash}X}@{}}
\toprule
\multirow{3.5}{*}{\textbf{Methods}} & \multicolumn{7}{c}{\textbf{Meta-World (Medium)}} & \multicolumn{2}{c}{\textbf{Meta-World (Hard)}} \\
\cmidrule(lr){2-8} \cmidrule(lr){9-10}
& \shortstack[c]{Coffee\\Push}
& \shortstack[c]{Hammer}
& \shortstack[c]{Peg Insert\\Side}
& \shortstack[c]{Push\\Wall}
& \shortstack[c]{Soccer}
& \shortstack[c]{Sweep}
& \shortstack[c]{Sweep\\Into}
& \shortstack[c]{Hand\\Insert}
& \shortstack[c]{Pick Out\\of Hole} \\
\midrule
Expert & \dd{100}{0} & \dd{100}{0} & \dd{91}{2} & \dd{100}{0} & \dd{91}{2} & \dd{100}{0} & \dd{91}{1} & \dd{100}{0} & \dd{100}{0} \\
\bottomrule
\end{tabularx}

\vspace{0.12em}

\begin{tabularx}{\textwidth}{@{}l*{3}{>{\centering\arraybackslash}X}*{5}{>{\centering\arraybackslash}X}>{\centering\arraybackslash}X@{}}
\toprule
\multirow{3.5}{*}{\textbf{Methods}} & \multicolumn{3}{c}{\textbf{Meta-World (Hard)}} & \multicolumn{5}{c}{\textbf{Meta-World (Very Hard)}} & \multirow{3}{*}{\textbf{Avg.}} \\
\cmidrule(lr){2-4} \cmidrule(lr){5-9}
& \shortstack[c]{Assembly}
& \shortstack[c]{Pick\\Place}
& \shortstack[c]{Push}
& \shortstack[c]{Shelf\\Place}
& \shortstack[c]{Disassemble}
& \shortstack[c]{Stick\\Pull}
& \shortstack[c]{Stick\\Push}
& \shortstack[c]{Pick Place\\Wall}
& \\
\midrule
Expert & \dd{100}{0} & \dd{100}{0} & \dd{99}{0} & \dd{99}{0} & \dd{95}{1} & \dd{95}{1} & \dd{100}{0} & \dd{100}{0} & \dd{96.7}{0.5} \\
\bottomrule
\end{tabularx}

\endgroup
\end{table}

\raggedbottom

\begin{table}[htbp]
\centering\small
\caption{Success rates (mean $\pm$ standard deviation across three seeds, \%). IDP w/o geometry retains both training inputs and sets $M_i=0$.}
\label{tab:factorial}
\setlength{\tabcolsep}{4pt}
\begin{tabular}{lrrrr}
\toprule
Task & \makecell{Direct input\\only} & \makecell{Direct input\\+ geometry} & IDP w/o geometry & IDP \\
\midrule
Door & $96.67\pm1.15$ & $96.33\pm0.58$ & $95.67\pm2.31$ & $97.67\pm2.08$ \\
Box-Close & $63.33\pm8.74$ & $58.50\pm6.50$ & $62.67\pm3.55$ & $67.17\pm6.45$ \\
Push & $44.67\pm4.16$ & $49.33\pm5.51$ & $50.67\pm6.51$ & $52.33\pm5.69$ \\
\bottomrule
\end{tabular}
\end{table}

\subsection{Geometry and Expert-Proximal Training}\label{app:matched}
Conditional expert geometry and expert-proximal supervision act jointly in the training objective. We independently enable these two components on Door, Box-Close, and Push to examine their individual and combined effects. Each configuration uses three training seeds with matched initialization and evaluation initial states. Each evaluation uses 200 episodes on Box-Close and 100 on Door and Push.

Combining both components gives the highest mean success on all three tasks. We also test the correspondence between each observation and its geometric weights by permuting the matrices $M_i$ across samples. This reduces success from 97.67\% to 93.00\% on Door and from 67.17\% to 60.83\% on Box-Close, supporting the use of geometry associated with the current observation.

\subsection{Analysis of Conditional Expert Geometry}\label{app:support}

\paragraph{Observation relevance.}
Observation similarity determines how much each expert action contributes to $G_i$. We examine whether highly similar observations correspond to related expert actions and object states. Using the ToolHang-State policy, we sample 10,240 observations from the training data. The most similar reference has a normalized expert-action distance of 1.622, compared with 5.879 for a random reference from the same batch. The corresponding object-state distances are 0.643 and 4.768. References from other demonstration trajectories receive 98.9\% of the total weight.

A separate encoder comparison keeps the observations and reference candidates fixed. The reference expert-action distance is 2.930 with the policy encoder, 3.263 with a frozen ImageNet encoder, and 4.344 with a random encoder. These results support using the policy's observation features to identify related expert actions.

\paragraph{Number of reference demonstrations.}
The geometry estimate uses a finite set of expert actions from the training batch. Reducing this set can change the directions emphasized by $M_i$ and the resulting training signal. We keep the policy parameters and evaluated observations fixed and vary only the number of geometry candidates, using 1,024 candidates as the reference. We measure relative changes in $I+M_i$, overlap between the emphasized coordinates, and changes in the gradients at the two training outputs (Figure~\ref{fig:geometry_support}).

\begin{figure}[ht]
\centering
\includegraphics[width=0.9\linewidth]{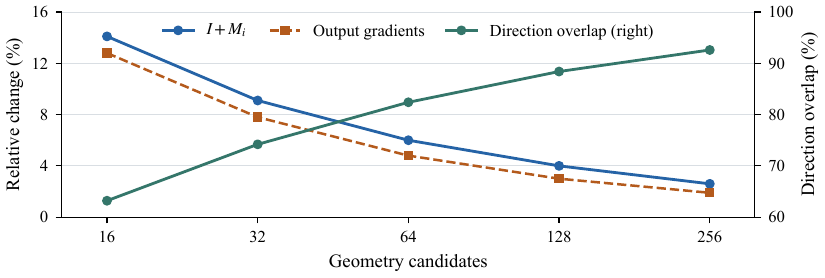}
\par\smallskip
{\small
\setlength{\tabcolsep}{4pt}
\begin{tabular}{lrrr}
\toprule
Reference removal & $I+M_i$ change (\%) & Direction overlap (\%) & Output gradient change (\%) \\
\midrule
Most similar reference & 0.03 & 99.9 & 0.02 \\
All similar references & 18.0 & 33.2 & 17.8 \\
\bottomrule
\end{tabular}}
\caption{Effect of the number of reference demonstrations on geometric supervision. Results are measured relative to 1,024 candidates with fixed policy parameters and observations. Changes in the weights and output gradients are reported as median relative $L_2$ differences. Direction overlap is the Jaccard similarity of coordinates with $m_{i,\ell,d}>0.1$. Gradient change is measured on the concatenated gradients at the two training outputs. The two reference-removal conditions are shown separately below the plot.}
\label{fig:geometry_support}
\end{figure}

With 64 candidates, the median relative changes in $I+M_i$ and the output gradients are 6.0\% and 4.8\%, respectively. Removing the most similar candidate changes $I+M_i$ by 0.03\%, whereas removing all similar candidates reduces direction overlap to 33.2\%. We also train policies on Adroit Door while fixing the optimization batch at 128. Using 32, 64, and 128 geometry candidates gives mean success rates of 92.33\%, 94.67\%, and 97.67\%, respectively, across three training seeds.

\paragraph{Visual observation changes.}
For image observations, visual features determine the expert-action weights used to construct $G_i$. Appearance changes can therefore alter the reference actions even when the robot and objects remain in the same state. We compare clean and perturbed images from fixed simulator states, measuring reference-set overlap and changes in predicted actions. Actual task-state changes provide a reference for interpreting the magnitude of these changes.

\begin{table}[t]
\centering\small
\caption{Visual perturbations at fixed task states. Policies trained for 150,000 steps are evaluated on 32 paired observations. The distractor condition uses nine valid pairs in which the task objects remain visible. Entries are IDP / IDP w/o geometry. Top-10 overlap is the Jaccard similarity of the reference sets. Normalized action change is $\lVert a'-a\rVert_2/\max(\lVert a\rVert_2,10^{-12})$, where $a$ is the clean-observation prediction.}
\label{tab:visual_perturb}
\setlength{\tabcolsep}{4pt}
\begin{tabular}{lrr}
\toprule
Change & Top-10 overlap & Normalized action change \\
\midrule
Gaussian corruption & 0.994 / 0.989 & 0.0032 / 0.0033 \\
Background recoloring & 1.000 / 1.000 & 0.0033 / 0.0032 \\
Distractor block & 1.000 / 1.000 & 0.0069 / 0.0073 \\
Lighting & 0.885 / 0.867 & 0.0791 / 0.0699 \\
Task-state change & 0.894 / 0.955 & 0.0400 / 0.0404 \\
\bottomrule
\end{tabular}
\end{table}

The tested image corruption, background changes, and distractors produce smaller action changes than the task-state reference. Lighting produces larger changes for both methods.


\section{Limitations}\label{app:limitations}

\paragraph{Mode coverage.}
Under balanced Pick Peach demonstrations, all evaluated methods learn a single placement mode. Extending IDP to represent multiple valid action modes is a direction for future work.

\paragraph{Related demonstrations.}
Conditional expert geometry depends on related demonstrations in the minibatch. Removing all similar candidates substantially changes the emphasized directions in the support experiment (Appendix~\ref{app:support}).

\paragraph{Coordinate-wise geometry.}
The diagonal $M_i$ assigns separate weights to action coordinates and time positions. A structured matrix could additionally model interactions between these coordinates.

\paragraph{Observation representations.}
On Pick Peach, qpos-based IDP achieves 30\% success, compared with 40\% for MIP. Similar joint positions can occur with different peach--gripper relationships, so useful geometry depends on observation features that distinguish task contexts.

\paragraph{Batch reference.}
The reference is estimated from each minibatch, so its composition affects the resulting weights. Running estimates and task-conditioned references are possible extensions for broader data and multi-task training.

\paragraph{Variation across tasks.}
Tasks close to the success ceiling leave little room for improvement. In the supplementary ToolHang-State evaluations, the differences in success and perturbation recovery between IDP and IDP w/o geometry fall within the observed variation across training seeds.


\end{document}